\documentclass{article}

\PassOptionsToPackage{numbers}{natbib}
\usepackage[final]{neurips}

\usepackage[utf8]{inputenc} 
\usepackage[T1]{fontenc}    
\usepackage{hyperref}       
\usepackage{url}            
\usepackage{booktabs}       
\usepackage{amsfonts}       
\usepackage{nicefrac}       
\usepackage{microtype}      
\usepackage{xspace}

\usepackage{enumitem}
\usepackage{lipsum}
\usepackage{amsmath}
\usepackage{amssymb}
\usepackage{amsthm}
\usepackage{bm}
\usepackage{dsfont}
\usepackage{graphicx}
\usepackage{wrapfig}
\usepackage{subcaption}
\usepackage[font=small]{caption}
\usepackage{adjustbox}

\usepackage{multicol,changepage}
\usepackage[noend]{algpseudocode}
\usepackage{algorithm}

\usepackage{tabularx,ragged2e}
\newcolumntype{L}[1]{>{\minwd l{#1}}l<{\endminwd}}
\newcolumntype{C}[1]{>{\minwd c{#1}}c<{\endminwd}}
\newcolumntype{R}[1]{>{\minwd r{#1}}r<{\endminwd}}

\usepackage{pifont}

\usepackage[dvipsnames]{xcolor}

\usepackage{float}
\floatstyle{plaintop}
\restylefloat{table}

\newif\ifcomments
\commentstrue
\ifcomments
  \newcommand{\colornote}[3]{{\color{#1}\bf{#2: #3}\normalfont}}
\else
  \newcommand{\colornote}[3]{}
\fi

\newcommand{\given}{\,|\,}

\newcommand{\R}{\mathbb{R}}

\renewcommand{\S}{\mathcal{S}}
\newcommand{\X}{\mathcal{X}}

\renewcommand{\L}{\mathcal{L}}
\newcommand{\A}{\mathcal{A}}

\newcommand{\G}{\mathcal{G}}

\newcommand{\parents}{\textrm{Pa}}
\newcommand{\lone}{\ell_1}

\newcommand\mydots{\ifmmode\ldots\else\makebox[1em][c]{.\hfil.\hfil.}\thinspace\fi}

 \newcommand{\ad}[1]{\textsc{#1}}
\newcommand{\env}[1]{\texttt{#1}}
\newcommand{\lowerl}{{\Large \textbf{$\lrcorner$} }}

\newtheorem{proposition}{Proposition}
\newtheorem{theorem}{Theorem}
\newtheorem{lemma}{Lemma}
\newtheorem{corollary}{Corollary}
\newtheorem{remark}{Remark}[section]

\makeatletter
\newenvironment{appdxTheorem}[1]
  {\count@\c@theorem
   \global\c@theorem#1 %
    \global\advance\c@theorem\m@ne
   \theorem}
  {\endproposition
   \global\c@theorem\count@}
\makeatother

\makeatletter
\newenvironment{appdxCorollary}[1]
  {\count@\c@corollary
   \global\c@corollary#1 %
    \global\advance\c@corollary\m@ne
   \corollary}
  {\endproposition
   \global\c@corollary\count@}
\makeatother

\usepackage{titlesec}
\titlespacing*{\section}
{0pt}{5pt plus 3pt minus 2pt}{4pt plus 2pt}
\titlespacing*{\subsection}
{0pt}{4pt plus 3pt minus 1pt}{3pt plus 2pt}

\newcommand{\methodName}{\textsc{MoCoDA}\xspace}
\newcommand{\methodNameFull}{Model-based Counterfactual Data Augmentation\xspace}

\title{
\methodName: \methodNameFull
}

\author{
Silviu Pitis\thanks{Correspondence to \texttt{spitis@cs.toronto.edu}}$^{\ \ 1}$ \ Elliot Creager$^{1}$ \ Ajay Mandlekar$^2$ \ Animesh Garg$^{1,2}$\\
$^1$University of Toronto and Vector Institute, $^2$NVIDIA
}

\usepackage{subfiles} 

\begin{document}

\maketitle

\begin{abstract}
	The number of states in a dynamic process is exponential in the number of objects, making  reinforcement learning (RL) difficult in complex, multi-object domains. 
	For agents to scale to the real world, they will need to react to and reason about unseen combinations of objects. We argue that the ability to recognize and use local factorization in transition dynamics is a key element in unlocking the power of multi-object reasoning.
	To this end, we	show that (1) known local structure in the environment transitions is sufficient for an exponential reduction in the sample complexity of training a dynamics model, and (2) a locally factored dynamics model provably generalizes out-of-distribution to unseen states and actions.
	Knowing the local structure also allows us to predict \textit{which} unseen states and actions this dynamics model will generalize to.
	We propose to leverage these observations in a novel \methodNameFull (\methodName) framework. \methodName applies a learned locally factored dynamics model to an augmented distribution of states and actions to generate counterfactual transitions for RL.
	\methodName works with a broader set of local structures than prior work and allows for direct control over the augmented training distribution.
    \setcounter{footnote}{0} 
	We show that \methodName enables RL agents to learn policies that generalize to unseen states and actions. We use \methodName to train an offline RL agent to solve an out-of-distribution robotics manipulation task on which standard offline RL algorithms fail.\footnote{Visualizations \& code available at \url{https://sites.google.com/view/mocoda-neurips-22/}}
\end{abstract}

\section{Introduction}

Modern reinforcement learning (RL) algorithms have demonstrated remarkable success in several different domains such as games~\cite{mnih2015human, silver2017mastering} and robotic manipulation~\cite{kalashnikov2018scalable, andrychowicz2020learning}. By repeatedly attempting a single task through trial-and-error, these algorithms can learn to collect useful experience and eventually solve the task of interest. 
However, designing agents that can generalize in \textit{off-task} and \textit{multi-task} settings remains an open and challenging research question. 
This is especially true in the offline and zero-shot settings, in which the training data might be unrelated to the target task, and may lack sufficient coverage over possible states.

One way to enable generalization in such cases is through structured representations of states, transition dynamics, or task spaces. These representations can be directly learned, sourced from known or learned abstractions over the state space, or derived from causal knowledge of the world. Symmetries present in such representations enable compositional generalization to new configurations of states or tasks, either by building the structure into the function approximator or algorithm \citep{kipf2019contrastive,veerapaneni2020entity,goyal2019recurrent,nangue2020boolean}, or by using the structure for data augmentation \citep{andrychowicz2017hindsight,laskin2020reinforcement,pitis2020counterfactual}.

In this paper, we extend past work on structure-driven data augmentation by using a locally factored model of the transition dynamics to generate counterfactual training distributions.
This enables agents to generalize beyond the support of their original training distribution, including to novel tasks where learning the optimal policy requires access to states never seen in the experience buffer.
Our key insight is that a learned dynamics model that accurately captures local causal structure (a ``locally factored'' dynamics model) will predictably exhibit good generalization performance outside the empirical training distribution. We propose \methodNameFull (\methodName), which generates an augmented state-action distribution  where its locally factored dynamics model is likely to perform well, then applies its dynamics model to generate new transition data. By training the agent's policy and value modules on this augmented dataset, they too learn to generalize well out-of-distribution. 
To ground this in an example, we consider how a US driver might use \methodName to adapt to driving on the left side of the road while on vacation in the UK (Figure \ref{fig_ad}). Given knowledge of the target task, we can even focus the augmented distribution on relevant areas of the state-action space (e.g., states with the car on the left side of the road). 

Our main contributions are:
\renewcommand{\theenumi}{\Alph{enumi}}
\begin{enumerate}[leftmargin=0.5cm]
    \item Our proposed method, \methodName, leverages a masked dynamics model for data-augmentation in locally-factored settings, which relaxes strong assumptions made by prior work on factored MDPs and counterfactual data augmentation.
    \item \methodName allows for direct control of the state-action distribution on which the agent trains; we show that controlling this distribution in a task relevant way can lead to improved performance.
	\item We demonstrate ``zero-shot'' generalization of a policy trained with \methodName to states that the agent has never seen. With \methodName, we train an offline RL agent to solve an out-of-distribution robotics manipulation task on which standard offline RL algorithms fail.
\end{enumerate}

\definecolor{oneorange}{RGB}{230,159,66}
\definecolor{twored}{RGB}{194,40,27}
\definecolor{threeblue}{RGB}{26,70,140}
\definecolor{fourgreen}{RGB}{143,194,78}

\begin{figure}[!t]
	\centering
    	\includegraphics[width=.9\textwidth]{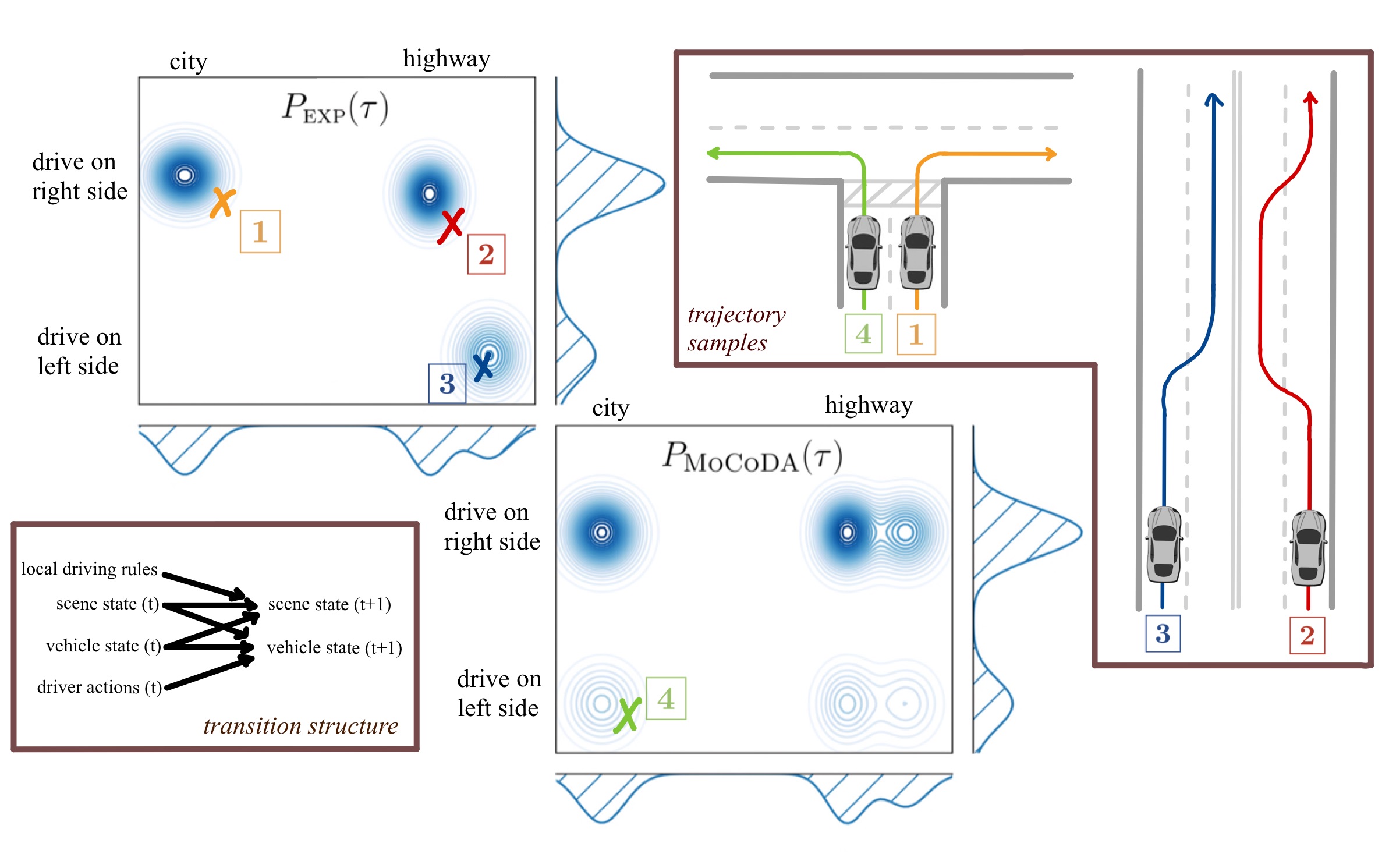}
	\caption{
	\textbf{Out-of-Distribution Generalization using \methodName}: A US driver can use \methodName to quickly adapt to driving in the left lane during a UK trip.	Their prior experience $P_{\ad{Emp}}(\tau)$ (\textbf{top left}) contains mostly right-driving experience (e.g. {\tiny \color{oneorange} \boxed{\textbf 1}}, {\tiny \color{twored} \boxed{\textbf 2}}) 
	and a limited amount of left-driving experience after renting the car in the UK (e.g. {\tiny \color{threeblue} \boxed{\textbf 3}}).
	A locally factored model that captures the transition structure (\textbf{bottom left}) allows the agent to accurately sample counterfactual experience from $P_{\ad{Mocoda}}(\tau)$ (\textbf{bottom center}), including novel left-lane city driving maneuvers (e.g. {\tiny \color{fourgreen} \boxed{\textbf 4}}).  This enables fast adaptation when learning an optimal policy for the new task (UK driving).
	Our framework \methodName draws single-step transition samples from $P_{\ad{Mocoda}}(\tau)$ given $P_{\ad{Emp}}(\tau)$ and knowledge of the causal structure; several realizations of this framework are described in Section \ref{sec:method}.
	}
	\label{fig_ad}
\end{figure}

\section{Preliminaries}

\subsection{Background}

We model the environment as an infinite-horizon, reward-free Markov Decision Process (MDP), described by tuple $\langle \S, \A, P, \gamma \rangle$ consisting of the state space, action space, transition function, and discount factor, respectively \citep{puterman2014markov,sutton2018reinforcement}. 
We use lowercase for generic instances and uppercase for variables (e.g., $s \in \textrm{range}(S) \subseteq \S$, though we also abuse notation and write $S \in \S$). 
A \textit{task} is defined as a tuple $\langle r, P_0 \rangle$, where $r: \S \times \A \to \R$ is a reward function and $P_0$ is an initial distribution over $S$.
The goal of the agent given a task is to learn a policy $\pi: \S \to \A$ that maximizes value $\mathbb{E}_{P,\pi}\sum_t\gamma^tr(s_t, a_t)$. 
Model-based RL is one approach to solving this problem, in which the agent learns a model $P_\theta$ of the transition dynamics $P$. The model is ``rolled out'' to generate ``imagined'' trajectories, which are used either for direct planning \citep{de2005tutorial,chaslot2008monte}, or as training data for the agent's policy and value functions \citep{sutton1991dyna,janner2019trust}.

\textbf{Factored MDPs}.\ \ A factored MDP (FMDP) is a type of MDP that assumes a globally factored transition model, which can be used to exponentially improve the sample complexity of RL \citep{guestrin2003efficient,kearns1999efficient,osband2014near}. In an FMDP, states and actions are described by a set of variables $\{X^i\}$, so that $\S\!\times\!\A = \X^1\!\times\!\X^2\!\times\!\dots\!\times\!\X^n$, and each state variable $X^i \in \X^i \ (\X^i\ \textrm{is a subspace of}\ \S)$ is dependent on a subset of state-action variables (its ``parents'' $\parents(X^i)$) at the prior timestep, $X^i \sim P_i(\parents(X^i))$. We call a set $\{X^j\}$ of state-action variables a ``parent set'' if there exists a state variable $X^i$ such that $\{X^j\} = \parents(X^i)$. We say that $X^i$ is a ``child'' of its parent set $\parents(X^i)$. We refer to the tuple $\langle X^i, \parents(X^i), P_i(\cdot) \rangle$ as a ``causal mechanism''. 

\textbf{Local Causal Models}.\ \  
Because the strict global factorization assumed by FMDPs is rare,
recent work on data augmentation for RL and object-oriented RL suggests that transition dynamics might be better understood in a local sense, where all objects may interact with each other over time, but in a locally sparse manner \citep{goyal2019recurrent,kipf2019contrastive,madan2021fast}. 
Our work uses an abridged version of the Local Causal Model (LCM) framework \cite{pitis2020counterfactual}, as follows:
We assume the state-action space decomposes into a disjoint union of local neighborhoods: $\S\!\times\!\A\ = \L_1 \sqcup \L_2 \sqcup \dots \sqcup \L_n$. A neighborhood $\L$ is associated with its own transition function $P^{\L}$, which is factored according to its graphical model $\G^\L$ \citep{koller2009probabilistic}. We assume no two graphical models share the same structure\footnote{
This assumption is a matter of convenience that makes counting local subspaces in Section \ref{sec:theory} slightly easier and simplifies our implementation of the locally factored dynamics model in Section \ref{sec:method}. To accommodate cases where subspaces with different dynamics share the same causal structure, one could identify local subspaces using a latent variable rather than the mask itself, which we leave for future work.
} (i.e., the structure of $\G^\L$ uniquely identifies $\L$). Then, analogously to FMDPs, if $(s_t, a_t) \in \L$, each state variable $X^i_{t+1}$ at the next time step is dependent on its parents $\parents^\L(X^i_{t+1})$ at the prior timestep, $X^i_{t+1} \sim P_i^\L(\parents^\L(X^i_{t+1}))$. We define mask function $M: \S \times \A \to \{\L_i\}$ that maps $(s, a) \in \L$ to the adjacency matrix of $\G^\L$. This formalism is summarized in Figure \ref{fig_formalism}, and differs from FMDPs in that each $\L$ has its own factorization.

\begin{figure}[!t]
	\centering
	\includegraphics[width=\textwidth]{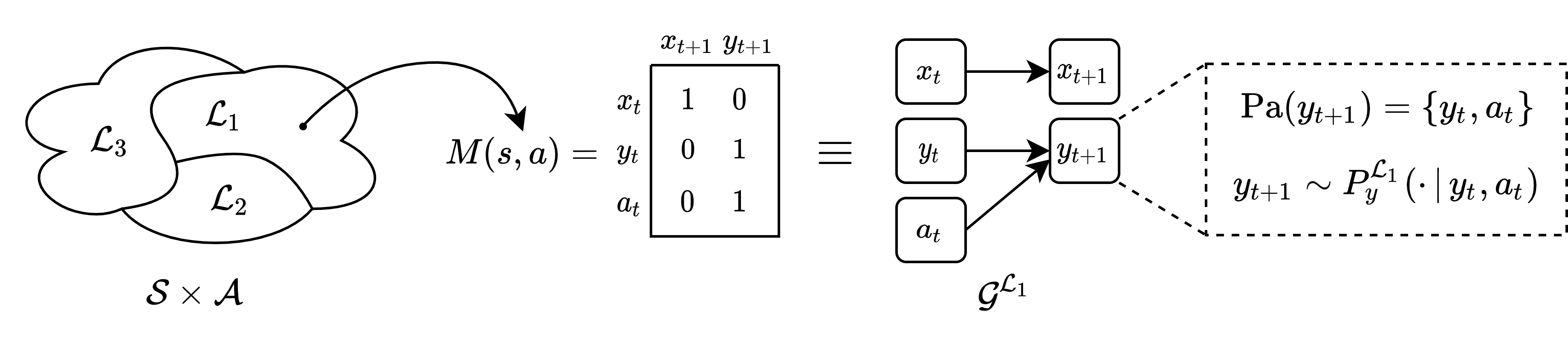}
	\vspace{-0.5cm}
	\caption{\textbf{Locally Factored Dynamics:} The state-action space $\S\times\A$ is divided into local subsets, $\L_1, \L_2, \L_3$, which each have their own factored causal structure, $\G^\L$. The local transition model $P^\L$ is factored according to $\G^\L$; e.g., in the example shown, $P^\L(x_t, y_t, a_t) = [P_x(x_t), P_y(y_t, a_t)]$.}
	\label{fig_formalism}
\end{figure}

Given knowledge of $M$, the Counterfactual Data Augmentation (CoDA) framework \cite{pitis2020counterfactual} allowed agents to stitch together empirical samples from disconnected causal mechanisms to derive novel transitions.
It did this by swapping compatible components between the observed transitions to create new ones, arguing that this procedure can generate exponentially more data samples as the number of disconnected causal components grows. CoDA was shown to significantly improve sample complexity in several settings, including the offline RL setting and a goal-conditioned robotics control setting. 
Because CoDA relied on empirical samples of the causal mechanisms to generate data in a model-free fashion, however, it required that the causal mechanisms be completely disentangled.
The proposed \methodName \textit{leverages a dynamics model to improve upon model-free CoDA} in several respects: (a) by using a learned dynamics model, \methodName works with overlapping parent sets, (b) by modeling the parent distribution, \methodName allows the agent to control the overall data distribution, (c) \methodName demonstrates zero-shot generalization to new areas of the state space, allowing the agent to solve tasks that are entirely outside the original data distribution.

\subsection{Related Work}

\textbf{RL with Structured Dynamics}.\ \  
A growing literature recognizes the advantages that structure can provide in RL, including both improved sample efficiency \cite{loynd2020working,balaji2020factoredrl,huang2022action} and generalization performance \cite{zhou2022policy,wang2018nervenet,sodhani2022improving}. Some of these works involve sparse interactions whose structure changes over time \cite{goyal2019recurrent,kipf2019contrastive}, which is similar to and inspires the locally factored setup assumed by this paper. Most existing work focuses on leveraging structure to improve the architecture and generalization capability of the function approximator \cite{zhou2022policy}. Although \methodName also uses the structure for purposes of improving the dynamics model, our proposed method is among the few existing works that also use the structure for data augmentation \cite{lu2020sample,mandlekar2020learning,pitis2020counterfactual}.

Several past and concurrent works aim to tackle unsupervised object detection \cite{locatello2020object,dittadi2021generalization} (i.e., learning an entity-oriented representation of states, which is a prerequisite for learning the dynamics factorization) and learning the dynamics factorization \cite{kipf2018neural,wang2022causal}. These are both open problems that run orthogonal to \methodName. We expect that as solutions for unsupervised object detection and factored dynamics discovery improve, \methodName will find broader applicability. 

\textbf{RL with Causal Dynamics}.\ \ 
Adopting this formalism allows one to cast several important problems within RL as questions of causal inference, such as off-policy evaluation \citep{buesing2018woulda,oberst2019counterfactual}, learning baselines for model-free RL \citep{mesnard2021counterfactual}, and policy transfer \citep{killian2022counterfactually}.
\citet{lu2020sample} applied SCM dynamics to data augmentation in continuous sample spaces, and discussed the conditions under which the generated transitions are uniquely identifiable counterfactual samples.
This approach models state and action variables as unstructured vectors, emphasizing benefit in modeling action interventions for settings such as clinical healthcare where exploratory policies cannot be directly deployed.
We take a complementary approach by modeling structure \emph{within} state and action variables, and our augmentation scheme involves sampling entire causal mechanisms (over multiple state or action dimensions) rather than action vectors only.
See Appendix \ref{sec:causal-appendix} for a more detailed discussion of how \methodName sampling relates to causal inference and counterfactual reasoning.

\section{Generalization Properties of Locally Factored Models}\label{sec:theory}

\subsection{Sample Complexity of Training a Locally Factored Dynamics Model}

In this subsection, we provide an original adaptation of an elementary result from model-based RL to the \textit{locally} factored setting, to show that factorization can exponentially improve sample complexity. We note that several theoretical works have shown that the FMDP structure can be exploited to obtain similarly strong sample complexity bounds in the FMDP setting. Our goal here is not to improve upon these results, but to adapt a small part (model-based generalization) to the significantly more general locally factored setting and show that local factorization is enough for (1) \textit{exponential gains in sample complexity} and (2) \textit{out-of-distribution generalization} with respect to the empirical joint, to a set of states and actions that may be exponentially larger than the empirical set. Note that the following discussion applies to tabular RL, but we apply our method to continuous domains. 

\textbf{Notation}.\ \ We work with finite state and action spaces ($|\S|, |\A| < \infty$) and assume that there are $m$ local subspaces $\L$ of size $|\L|$, such that $m|\L| = |\S||\A|$. For each subspace $\L$, we assume transitions factor into $k$ causal mechanisms $\{P_i\}$, each with the same number of possible children, $|c_i|$, and the same number of possible parents, $|\parents_i|$. Note $m\Pi_i|c_i| = |\S|$  (child sets are mutually exclusive) but $m\Pi_i|\parents_i| \geq |\S||\A|$ (parent sets may overlap).

\begin{theorem}\label{theorem_one}
Let $n$ be the number of empirical samples used to train the model of each local causal mechanism, $P_{i, \theta}^\L$ at each configuration of parents $\parents_i = x$. There exists constant $c$ such that, if
$$
n \geq \frac{ck^2|c_i|\log(|\S||\A|/\delta)}{\epsilon^2},
$$
then, with probability at least $1-\delta$, we have:
$$\max_{(s, a)} \Vert P(s, a) - P_{\theta}(s, a) \Vert_1 \leq  \epsilon.$$
\end{theorem}

\begin{proof}[Sketch of Proof]
We apply a concentration inequality to bound the $\ell_1$ error for fixed parents and extend this to a bound on the $\ell_1$ error for a fixed $(s, a)$ pair. The conclusion follows by a union bound across all states and actions. See Appendix \ref{appdx_proposition} for details. 
\end{proof}

To compare to full-state dynamics modeling, we can translate the sample complexity from the per-parent count $n$ to a total count $N$. Recall $m\Pi_i|c_i| = |\S|$, so that $|c_i| = (|\S|/m)^{1/k}$, and $m\Pi_i|\parents_i| \geq |\S||\A|$. We assume a small constant overlap factor $v \geq 1$, so that $|\parents_i| = v(|\S||\A|/m)^{1/k}$. We need the total number of component visits to be $n|\parents_i|km$, for a total of $nv(|\S||\A|/m)^{1/k}m$ state-action visits, assuming that parent set visits are allocated evenly, and noting that each state-action visit provides $k$ parent set visits. This gives:

\begin{corollary}\label{corollary_one}
To bound the error as above, we need to have
$$N \geq \frac{cmk^2(|\S|^2|\A|/m^2)^{1/k}\log(|\S||\A|/\delta)}{\epsilon^2},$$
total train samples, where we have absorbed the overlap factor $v$ into constant $c$.
\end{corollary}

Comparing this to the analogous bound for full-state model learning (\citet{agarwal2019reinforcement}, Prop. 2.1): 

$$N \geq \frac{c|\S|^2|\A|\log(|\S||\A|/\delta)}{\epsilon^2},$$

we see that we have gone from super-linear $O(|\S|^2|\A|\log(|\S||\A|))$ sample complexity in terms of $|S||A|$, to the exponentially smaller $O(mk^2(|\S|^2|\A|/m^2)^{1/k}\log(|\S||\A|))$. 

This result implies that \textit{for large enough $|\S||\A|$ our model \textit{must} generalize to unseen states and actions}, since the number of samples needed ($N$) is exponentially smaller than the size of the state-action space ($|\S||\A|$). In contrast, if it did not, then sample complexity would be $\Omega(|\S||\A|)$. 

\begin{remark}
The global factorization property of FMDPs is a strict assumption that rarely holds in reality. Although local factorization is broadly applicable and significantly more realistic than the FMDP setting, it is not without cost. In FMDPs, we have a single subspace ($m=1$). In the locally factored case, the number of subspaces $m$ is likely to grow exponentially with the number of factors $k$, as there are exponentially many ways that $k$ factors can interact. To be more precise, there are $k2^k$ possible bipartite graphs from $k$ nodes to $k$ nodes. Nevertheless, by comparing bases ($2 \ll |\S||\A|$), we see that we still obtain exponential gains in sample complexity from the locally factored approach.
\end{remark}

\subsection{Training Value Functions and Policies for Out-of-Distribution Generalization}

In the previous subsection, we saw that a locally factored dynamics model provably generalizes outside of the empirical joint distribution. A natural question is whether such \textit{local factorization can be leveraged to obtain similar results for value functions and policies}?

We will show that the answer is \textit{yes}, but perhaps counter-intuitively, it is not achieved by directly training the value function and policy on the empirical distribution, as is the case for the dynamics model. The difference arises because learned value functions, and consequently learned policies, involve the long horizon prediction $\mathbb{E}_{P,\pi}\sum_{t=0}^{\infty}\gamma^tr(s_t, a_t)$, which may not benefit from the local sparsity of $\G^\L$. When compounded over time, sparse local structures can quickly produce an entangled long horizon structure (cf. the ``butterfly effect''). 
Intuitively, even if several pool balls are far apart and locally disentangled, future collisions are central to planning and the optimal policy depends on the relative positions of all balls. This applies even if rewards are factored (e.g., rewards in most pool variants) \citep{sodhani2022improving}. 

We note that, although temporal entanglement may be exponential in the branching factor of the unrolled causal graph, it's possible for the long horizon structure to stay sparse (e.g., $k$ independent factors that never interact, or long-horizon disentanglement between descision relevant and decision irrelevant variables \cite{huang2022action}). It's also possible that other regularities in the data will allow for good out-of-distribution generalization. Thus, we cannot claim that value functions and policies will never generalize well out-of-distribution (see \citet{veerapaneni2020entity} for an example when they do). Nevertheless, we hypothesize that exponentially fast entanglement does occur in complex natural systems, making direct generalization of long horizon predictions difficult. 

Out-of-distribution generalization of the policy and value function can be achieved, however, by leveraging the generalization properties of a locally factored dynamics model. We propose to do this by generating out-of-distribution states and actions (the ``parent distribution''), and then applying our learned dynamics model to generate transitions that are used to train the policy and value function. We call this process \methodNameFull (\methodName). 

\section{
\methodNameFull
}\label{sec:method}

\begin{figure}[!t]
  \centering
  \includegraphics[width=0.72\textwidth]{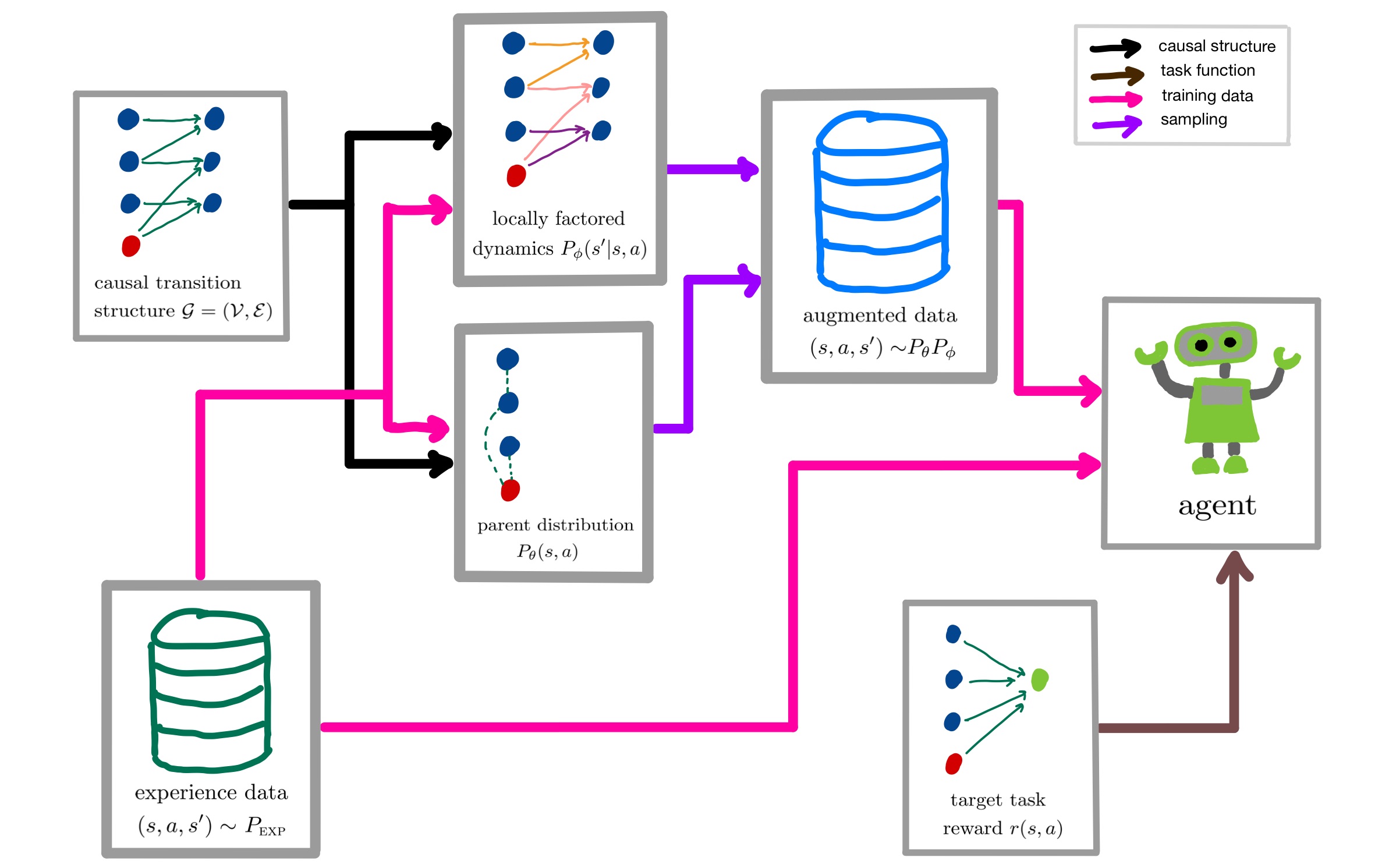}
\caption{\textbf{Training an RL Agent with \methodName}: We use the empirical dataset to train parent distribution model, $P_\theta(s, a)$ and locally factored dynamics model $P_\phi(s'\given s, a)$, both informed by the local structure. The dynamics model is applied to the parent distribution to produce augmented dataset $P_\theta P_\phi$. The augmented \& empirical datasets are labeled with the target task reward, $r(s, a)$ and fed into the RL algorithm as training data.}
    \label{fig:block-diagram}
\vspace{0.3\baselineskip}
\end{figure}

In the previous section, we discussed how locally factored dynamics model can generalize beyond the empirical dataset to provide accurate predictions on an augmented state-action distribution we call the ``parent distribution''.
We now seek to leverage this out-of-distribution generalization in the dynamics model to bootstrap the training of an RL agent.
Our approach is to control the agent's training distribution $P(s, a, s')$ via the locally factored dynamics $P_\phi(s'|s,a)$ and the  parent distribution $P_\theta(s,a)$ (both trained using experience data).
This allows us to sample \emph{augmented} transitions (perhaps unseen in the experience data) for consumption by a downstream RL agent.
We call this framework \methodName, and summarize it using the following three-step process:

\begin{enumerate}
    \item[\textbf{Step 1}\ ] Given known parent sets, generate appropriate parent distribution $P_\theta(s, a)$.
    \item[\textbf{Step 2}\ ] Apply a learned dynamics model $P_\phi(s'|s,a)$ to parent distribution to generate ``augmented dataset'' of transitions $(s, a, s')$.
    \item[\textbf{Step 3}\ ] Use augmented dataset $s,a,s' \sim P_\theta P_\phi$ (alongside experience data, if desired) to train an off-policy RL agent on the (perhaps novel) target task. 
\end{enumerate}

Figure \ref{fig:block-diagram} illustrates this framework in a block diagram. An instance of \methodName is realized by specific choices at each step. For example, the original CoDA method \citep{pitis2020counterfactual} is an instance of \methodName, which (1) generates the parent distribution by uniformly swapping non-overlapping parent sets, and (2) uses subsamples of empirical transitions as a locally factored dynamics model. CoDA works when local graphs have non-overlapping parent sets, but it does not allow for control over the parent distribution and does not work in cases where parent sets overlap. \methodName generalizes CoDA, alleviating these restrictions and allowing for significantly more design choices, discussed next. 

\subsection{Generating the Parent Distribution}\label{subsection_generating_ad}

What parent distribution (Step 1) should be used to generate the augmented dataset? We describe some options below, noting that our proposals (\ad{Mocoda}, \ad{Mocoda-U}, \ad{Mocoda-P}) rely on knowledge of (possibly local) parent sets---i.e., they require the state to be decomposed into objects. 

\textbf{Baseline Distributions.}\ \ \ \ If we restrict ourselves to states and actions in the empirical dataset (\textbf{\ad{Emp}}) or short-horizon rollouts that start in the empirical state-action distribution (\textbf{\ad{Dyna}}), as is typical in Dyna-style approaches \citep{sutton2018reinforcement,janner2019trust}, we limit ourselves to a small neighborhood of the empirical state-action distribution. This forgoes the opportunity to train our off-policy RL agent on out-of-distribution data that may be necessary for learning the target task.

Another option is to sample random state-actions from $\S\times\A$ (\textbf{\ad{Rand}}). While this provides coverage of all state-actions relevant to the target task, there is no guarantee that our locally factorized model generalizes well in \ad{Rand}. The proof of Theorem \ref{theorem_one} shows that our model only generalizes well to a particular $(s, a)$ if each component generalizes well on the configurations of each parent set in that $(s, a)$. In context of Theorem \ref{theorem_one}, this occurs only if the empirical data used to train our model contained at least $n$ samples for each set of parents in $(s, a)$. This suggests focusing on data whose parent sets have sufficient support in the empirical dataset. 

\textbf{The \ad{Mocoda} distribution.}\ \ \ \ We do this by constraining the marginal distribution of each parent set (within local neighborhood $\L$) in the augmented distribution to match the corresponding marginal in the empirical dataset. As there are many such distributions, in absence of additional information, it is sensible to choose the one with maximum entropy \citep{jaynes1957information}. We call this maximum entropy, marginal matching distribution the \textbf{\ad{Mocoda}} augmented distribution. Figure \ref{fig_ad} provides an illustrative example of going from \ad{Emp} (driving primarily on the right side) to \ad{Mocoda} (driving on both right and left). We propose an efficient way to generate the $\ad{Mocoda}$ distribution using a set of Gaussian Mixture Models, one for each parent set distribution. We sample parent sets one at a time, conditioning on any previous partial samples due to overlap between parent sets. This process is detailed in Appendix \ref{impl_details}.

\textbf{Weaknesses of the \ad{Mocoda} distribution.}\ \ \ \ Although our locally factored dynamics model is likely to generalize well on \ad{Mocoda}, there are a few reasons why training our RL agent on $\ad{Mocoda}$ in Step 3 may yield poor results. First, if there are empirical imbalances within parent sets (some parent configurations more common than others), these imbalances will appear in $\ad{Mocoda}$. Moreover, multiple such imbalances will compound exponentially, so that $(s, a)$ tuples with rare parent combinations will be extremely rare in \ad{Mocoda}, even if the model generalizes well to them. 
Second, $\textrm{Support}(\ad{Mocoda})$ may be so large that it makes training the RL algorithm in Step 3 inefficient. Finally, the cost function used in RL algorithms is typically an expectation over the training distribution, and optimizing the agent in irrelevant areas of the state-action space may hurt performance. The above limitations suggest that rebalancing $\ad{Mocoda}$ might improve results.

\textbf{\ad{Mocoda-U} and \ad{Mocoda-P}.}\ \ \ \ To mitigate the first weakness of \ad{Mocoda} we might skew $\ad{Mocoda}$ toward the uniform distribution over its support, $\mathcal{U}(\textrm{Support}(\ad{Mocoda}))$. Although this is possible to implement using rejection sampling when $k$ is small, exponential imbalance makes it impractical when $k$ is large. A more efficient implementation reweights the GMM components used in our \ad{Mocoda} sampler. We call this approach (regardless of implementation) $\textbf{\ad{Mocoda-U}}$. To mitigate the second and third weaknesses of \ad{Mocoda}, we need additional knowledge about the target task---e.g., domain knowledge or expert trajectories. We can use such information to define a prioritized parent distribution \textbf{\ad{Mocoda-P}} with support in \textrm{Support}({\ad{Mocoda}}), which can also be obtained via rejection sampling (perhaps on \ad{Mocoda-U} to also relieve the initial imbalance). 

\subsection{The Choice of Dynamics Model and RL Algorithm}

Once we have a parent distribution, $P_\theta(s, a)$, we generate our augmented dataset by applying dynamics model $P_\phi(s'\given s, a)$. The natural choice in light of the discussion in Section \ref{sec:theory} is a locally factored model. This requires knowledge of the local factorization, which is more involved than the parent set knowledge used to generate the \ad{Mocoda} distribution and its reweighted variants. We note, however, that a locally factored model may not be strictly necessary for \methodName, so long as the underlying dynamics are factored. Although unfactored models do not perform well in our experiments, we hypothesize that a good model with enough in-distribution data and the right regularization might learn to implicitly respect the local factorization. The choice of model architecture is not core to our work, and we leave exploration of this possibility to future work. 

\textbf{Masked Dynamics Model.}\ \ \ \ In our experiments, we assume access to a mask function $M: \S \times \A \to \{0, 1\}^{(|\S|+|\A|)\times |\S|}$ (perhaps learned \citep{kipf2018neural,pitis2020counterfactual}), which maps states and actions to the adjacency map of the local graph $\G^\L$. Given this mask function, we design a dynamics model $P_\phi$ that accepts $M(s, a)$ as an additional input and respects the causal relations in the mask (i.e., mutual information $I(X^i_t; X^j_{t+1} \given (S_t, A_t)\setminus X^i_t) = 0$ if $M(s_t, a_t)_{ij} = 0$). There are many architectures that enforce this constraint. In our experiments we opt for a simple one, which first embeds each of the $k$ parent sets: $f = [f_i(\parents_i)]_{i=1}^k$, and then computes the $j$-th child as a function of the sum of the masked embeddings, $g_j(M(s, a)_{\cdot,j}\cdot f)$. See Appendix \ref{impl_details} for further implementation details.

\textbf{The RL Algorithm.}\ \ \ \ After generating an augmented dataset by applying our dynamics model to the augmented distribution, we label the data with our target task reward and use the result to train an RL agent. \methodName works with a wide range of algorithms, and the choice of algorithm will depend on the task setting. For example, our experiments are done in an offline setup, where the agent is given a buffer of empirical data, with no opportunity to explore. For this reason, it makes sense to use offline RL algorithms, as this setting has proven challenging for standard online algorithms \citep{levine2020offline}. 

\begin{remark}
 The rationales for (1) regularizing the policy toward the empirical distribution in offline RL algorithms, and (2) training on the \ad{Mocoda} distribution, are compatible: in each case, we want to restrict ourselves to state-actions where our models generalize well. By using \ad{Mocoda} we expand this set \textit{beyond} the empirical distribution. Thus, when we apply offline RL algorithms in our experiments, we train their offline component (e.g., the action sampler in \ad{BCQ} \cite{fujimoto2019off} or the BC constraint in \ad{TD3-BC} \cite{fujimoto2021minimalist}) on the expanded \ad{Mocoda} training distribution.
\end{remark}

\section{Experiments}\label{sec:empirical}

\textbf{Hypotheses}\ \ \ \ Our experiments are aimed at finding support for two critical hypotheses:

\begin{enumerate}
    \item[\textbf{H1}\ \ ] Dynamics models, especially ones sensitive to the local factorization, are able to generalize well in the \ad{Mocoda} distribution.
    \item[\textbf{H2}\ \ ] This out-of-distribution generalization can be leveraged via data augmentation to train an RL agent to solve out-of-distribution tasks.
\end{enumerate}

Note that support for \textbf{H2} provides implicit support for \textbf{H1}.

\textbf{Domains}\ \ \ \ We test \methodName on two continuous control domains. First is a simple, but controlled, \env{2D Navigation} domain, where the agent must travel from one point in a square arena to another. States are 2D $(x, y)$ coordinates and actions are 2D $(\Delta x, \Delta y)$ vectors. In most of the state space, the sub-actions $\Delta x$ and $\Delta y$ affect only their respective coordinate. In the top right quadrant, however, the $\Delta x$ and $\Delta y$ sub-actions each affect \textit{both} $x$ and $y$ coordinates, so that the environment is locally factored. The agent has access to empirical training data consisting of left-to-right and bottom-to-top trajectories that are restricted to a \lowerl shape of the state space (see the \ad{Emp} distribution in Figure \ref{fig_toy_visualization}). We consider a target task where the agent must move from the bottom left to the top right. In this task there is sufficient empirical data to solve the task by following the $\lowerl$ shape of the data, but learning the optimal policy of going directly via the diagonal requires out-of-distribution generalization.

\begin{wrapfigure}{r}{0.23\textwidth}
    \vspace{-\baselineskip}
	\centering
    \captionsetup{width=0.22\textwidth}
	\includegraphics[width=0.22\textwidth,height=0.17\textwidth]{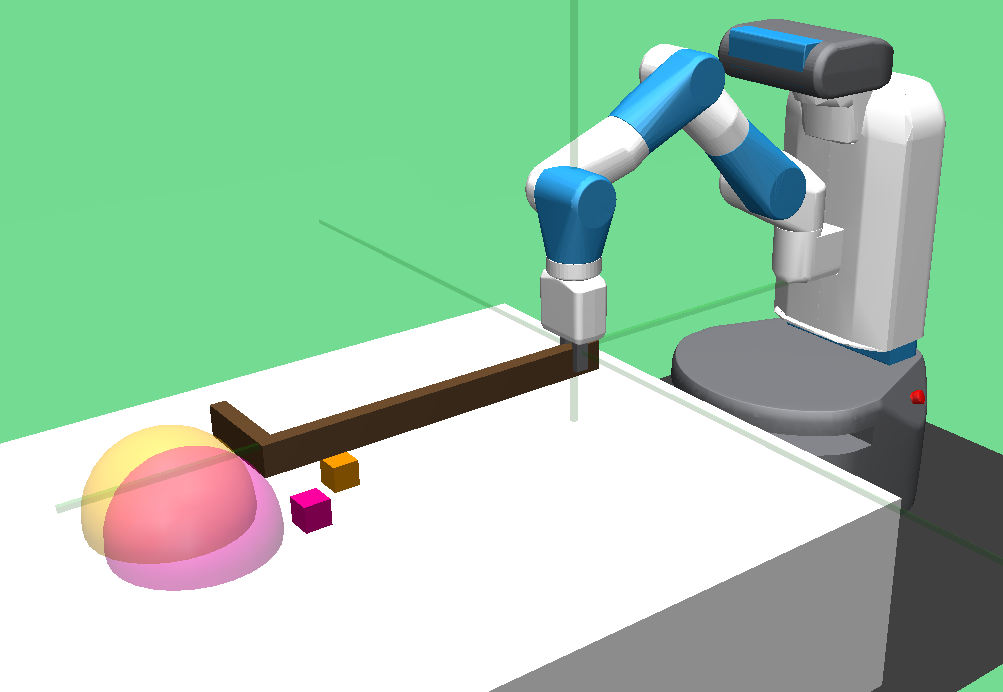}
	\label{fig_env_fetch}
    \vspace{-\baselineskip}
\end{wrapfigure}
Second, we test \methodName in a challenging \env{HookSweep2} robotics domain based on Hook-Sweep \citep{kurenkov2020ac}, in which a Fetch robot must use a long hook to sweep two boxes to one side of the table (either toward or away from the agent). The boxes are initialized near the center of the table, and the empirical data contains trajectories of the agent sweeping exactly one box to one side of the table, leaving the other in the center. The target task requires the agent to generalize to states that it has never seen before (both boxes together on one side of the table). This is particularly challenging because the setup is entirely offline (no exploration), where poor out-of-distribution generalization typically requires special offline RL algorithms that constrain the agent's policy to the empirical distribution \citep{levine2020offline,agarwal2020optimistic,kumar2020conservative,fujimoto2021minimalist}.

\newcommand{\buffer}{\hspace{1.7cm}}
\begin{figure}[!b]
	\centering
	\includegraphics[width=\textwidth]{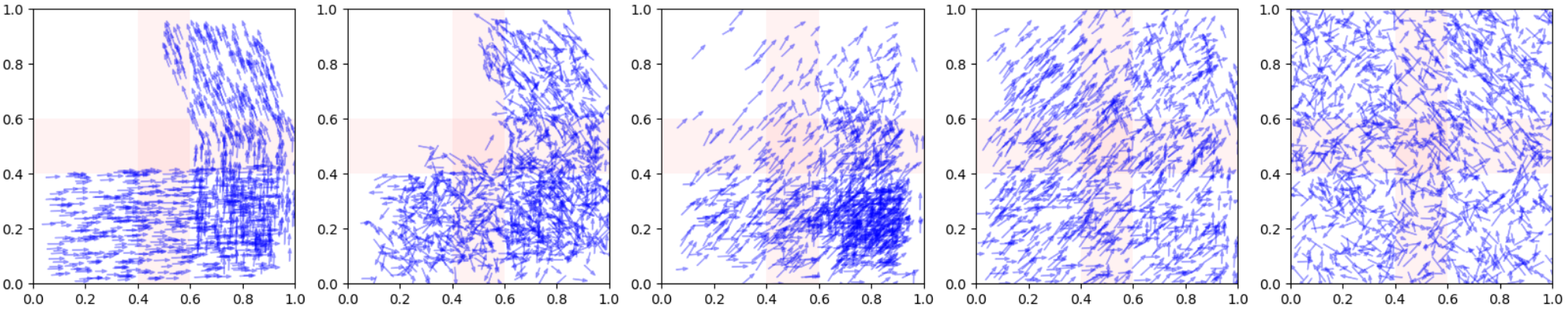}\\
	{\small\hspace{0.4cm}\ad{Emp}\hspace{0.4cm}\buffer\ad{Dyna}\buffer\ad{Mocoda}\hspace{-0.4cm}\buffer\ad{Mocoda-U}\hspace{0.05cm}\buffer\ad{Rand}\hfill}\\[6pt]
	\caption{\textbf{2D Navigation Visualization.} (Best viewed with 2x zoom) Blue arrows represent transition samples as a vector from $(x_t, y_t)$ to $(x_{t+1}, y_{t+1})$. Shaded red areas mark the edges of the initial states of empirical trajectories and the center of the square. We see that 5-step rollouts (\ad{Dyna}) do not fill in the center (needed for optimal policy), and fail to constrain actions to those that the model generalizes well on. For \ad{Mocoda}, we see the effect of compounding dataset imbalance discussed in Subsection \ref{subsection_generating_ad}, which is resolved by \ad{Mocoda-U}.}
	\label{fig_toy_visualization}
\vspace{-0\baselineskip}
\end{figure}

{\small
\tabcolsep=0pt\def\arraystretch{1.3}
\begin{table*}[!t]
\caption{\textbf{2D Navigation Dynamics Modeling Results:} Mean squared error $\pm$ 
\label{tab:toy-mse}
std. dev. over 5 seeds, scaled by 1e2 for clarity (best model boldfaced). The locally factored model experienced less performance degradation out-of-distribution, and performed better on all distributions, except for the empirical distribution (\ad{Emp}) itself.}\label{fig_toy_mse_results}
\newcommand{\smol}[1]{{\scriptsize\texttt{#1}}} 
\centering\small
\begin{tabularx}{\textwidth}{>{\centering\arraybackslash}p{3.2cm}| *5{>{\Centering}X}}\toprule	

&\multicolumn{5}{c}{\textbf{Generalization Error (MSE $\times 1e2$)} (lower is better)}\\
\cline{2-6}
Model Architecture&   \ad{Emp}  &  \ad{Dyna} &  \ad{Rand} & \ad{\textbf{MoCoDA}}  &  \ad{\textbf{MoCoDA-U}}   \\
	\midrule
Not Factored & \textbf{0.14 $\pm$ 0.04} &  2.41 $\pm$ 0.29 &  4.4 $\pm$ 0.31  & 0.95 $\pm$ 0.06 &  1.29 $\pm$ 0.15  \\
Globally Factored & 0.36 $\pm$ 0.01 &  2.09 $\pm$ 0.28  &  3.17 $\pm$ 0.3 &  0.41 $\pm$ 0.02 &  0.51 $\pm$ 0.02 \\
Locally Factored & 0.23 $\pm$ 0.1 &  \textbf{1.47 $\pm$ 0.2}7   &  \textbf{2.03 $\pm$ 0.19} &  \textbf{0.33 $\pm$ 0.11} &  \textbf{0.46 $\pm$ 0.11} \\
	\bottomrule
\end{tabularx}
\vspace{0.5\baselineskip}
\end{table*}
}

{\small
\tabcolsep=0pt\def\arraystretch{1.3}
\begin{table*}[!t]
\caption{\textbf{2D Navigation Offline RL Results:} Average steps to completion $\pm$ std. dev. over 5 seeds for various RL algorithms (best distribution in each row boldfaced), where average steps was computed over the last 50 training epochs. Training on \ad{Mocoda} and \ad{Mocoda-U} improved performance in all cases. Interestingly, even using \ad{Rand} improves performance, indicating the importance of training on out-of-distribution data. Note that this is an offline RL task, and so \ad{SAC} (an algorithm designed for online RL) is not expected to perform well.}
\label{fig_toy_batchrl_results}

\newcommand{\smol}[1]{{\scriptsize\texttt{#1}}}
\centering\small
\begin{tabularx}{\textwidth}{>{\centering\arraybackslash}p{3.1cm}| *5{>{\Centering}X}}\toprule	

\multicolumn{5}{c}{\textbf{Average Steps to Completion} (lower is better)}\\
\cline{2-6}
RL Algorithm&   \ad{Emp}  &  \ad{Rand} & \ad{\textbf{MoCoDA}}  &  \ad{\textbf{MoCoDA-U}}  & \ad{CoDA} \cite{pitis2020counterfactual} \\
	\midrule
\env{SAC} (online RL) & 53.1 $\pm$ 9.8  &  \textbf{27.6 $\pm$ 1.1} &  38.8 $\pm$ 18.3 &  41.3 $\pm$ 17.7 & 35.1 $\pm$ 18.1\\
\env{BCQ} & 58.5 $\pm$ 10.1  &  31.7 $\pm$ 2.4 &   \textbf{22.8 $\pm$ 0.4} &  24.8 $\pm$ 4.2 & 25.0 $\pm$ 0.4 \\
\env{CQL} & 45.8 $\pm$ 4.0 &  27.6 $\pm$ 1.3 &  22.8 $\pm$ 0.2 &  \textbf{22.7 $\pm$ 0.3} & 23.6 $\pm$ 0.5 \\
\env{TD3-BC} & 40.0 $\pm$ 16.1 &  26.1 $\pm$ 0.8 &  21.0 $\pm$ 0.7 &  \textbf{20.7 $\pm$ 0.8} & 21.4 $\pm$ 0.6 \\
	\bottomrule
\end{tabularx}
\vspace{1.2\baselineskip}
\end{table*}
}

\textbf{Directly comparing model generalization error.}\ \ \ \ In the \env{2D Navigation} domain we have access to the ground truth dynamics, which allows us to directly compare generalization error on variety of distributions, visualized in Figure \ref{fig_toy_visualization}. We compare three different model architectures: unfactored, globally factored (assuming that the $(x, \Delta x)$ and $(y, \Delta y)$ causal mechanisms are independent everywhere, which is not true in the top right quadrant), and locally factored. The models are each trained on a empirical dataset of 35000 transitions for up to 600 epochs, which is early stopped using a validation set of 5000 transitions. The results are shown in Table \ref{fig_toy_mse_results}. We find strong support for \textbf{H1}: even given the simple dynamics of \env{2d Navigation}, it is clear that the locally factored model is able to generalize better than a fully connected model, particularly on the \ad{Mocoda} distribution, where performance degradation is minimal. We note that the \ad{Dyna} distribution was formed by starting in \ad{Emp} and doing 5-step rollouts with \textit{random} actions. The random actions produce out-of-distribution data to which no model (not even the locally factored model) can generalize well to.

\vspace{0.2\baselineskip}
\textbf{Solving out-of-distribution tasks.}\ \ \ \ We apply the trained dynamics models to several base distributions and compare the performance of RL agents trained on each dataset. To ensure improvements are due to the augmented dataset and not agent architecture, we train several different algorithms, including: \env{SAC} \citep{haarnoja2018soft}, \env{BCQ} \citep{fujimoto2019off} (with DDPG \citep{lillicrap2015continuous}), \env{CQL} \citep{kumar2020conservative} and \env{TD3-BC} \citep{fujimoto2021minimalist}. 

The results on \env{2D Navigation} are shown in Table \ref{fig_toy_batchrl_results}. We see that for all algorithms, the use of the \ad{Mocoda} and \ad{Mocoda-U} augmented datasets greatly improve the average step count, providing support for \textbf{H2} and suggesting that using these datasets allows the agents to learn to traverse the diagonal of the state space, even though it is out-of-distribution with respect to \ad{Emp}. This is consistent with a qualitative assessment of the learned policies, which confirms that agents trained on the \lowerl\hspace{-0.12cm}-shaped \ad{Emp} distribution learn a \lowerl\hspace{-0.12cm}-shaped policy, whereas agents trained on \ad{Mocoda} and \ad{Mocoda-U} learn the optimal (diagonal) policy.

The results on the more complex \env{HookSweep2} environment, shown in Table \ref{table_fetch_batchrl_results}, provide further support for \textbf{H2}. On this environment, only results for \env{BCQ} and \env{TD3-BC} are shown, as the other algorithms failed on all datasets. For \env{HookSweep2} we used a prioritized \ad{Mocoda-P} parent distribution, as follows: knowing that the target task involves placing two blocks, we applied rejection sampling to \ad{Mocoda} to make the marginal distribution of the joint block positions approximately uniform over its support. The effect is to have good representation in all areas of the most important state features for the target task (the block positions). The visualization in Figure \ref{fig_fetch_parent_distributions} makes clear why training on \ad{Mocoda} or \ad{Mocoda-P} was necessary in order to solve this task: the base \ad{Emp} distribution simply does not have sufficient coverage of the goal space.  

\vfill

\begin{figure}[!t]
	\centering
	\includegraphics[width=0.9\textwidth]{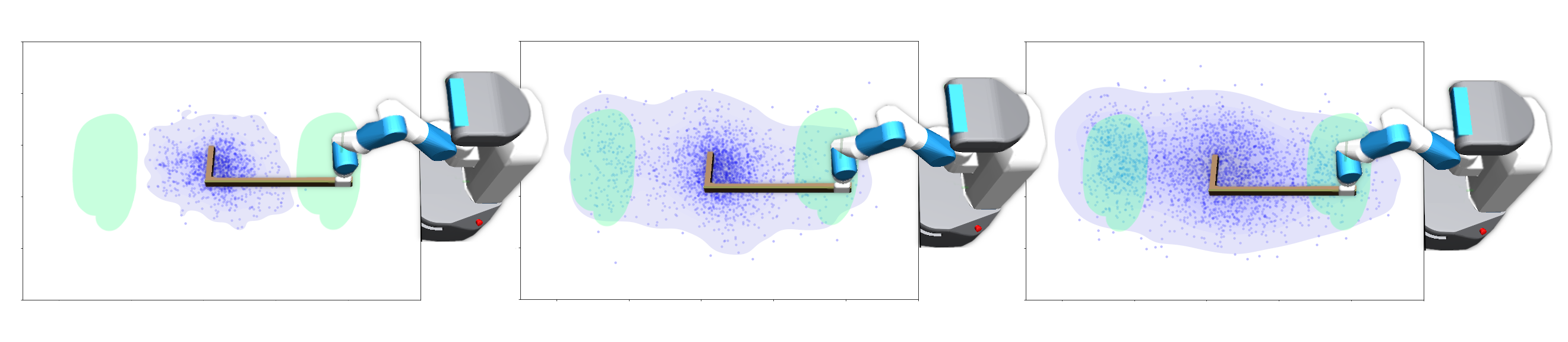}
	\vspace{-0.5\baselineskip}
	\caption{\textbf{HookSweep2 Visualization:} Stylized visualization of the distributions \ad{Emp} (left), \ad{Mocoda} (center), and \ad{Mocoda-P} (right). Each figure can be understood as a top down view of the table, where a point is a plotted if the two blocks are close together on the table. The distribution \ad{Emp} does not overlap with the green goal areas on the left and right, and so the agent is unable to learn. In the \ad{Mocoda} distribution, the agent gets some success examples. In the \ad{Mocoda-P} distribution, state-actions are reweighted so that the joint distribution of the two block positions is approximately uniform, leading to more evenly distributed coverage of the table.}
	\label{fig_fetch_parent_distributions}
	\vspace{-0\baselineskip}
\end{figure}

{\small
\tabcolsep=0pt\def\arraystretch{1.3}
\begin{table*}[!t]
\newcommand{\smol}[1]{{\scriptsize\texttt{#1}}}
\centering\small
\begin{tabularx}{\textwidth}{>{\centering\arraybackslash}p{3cm}| *5{>{\Centering}X}}
\toprule	
&\multicolumn{5}{c}{\textbf{Average Success Rate} (higher is better)}\\
\cline{2-6}
RLAlgorithm&  \vspace{-3px} \ad{Emp}  &\vspace{-3px}  \ad{\textbf{MoCoDA}}  & \vspace{-3px}  \ad{\textbf{MoCoDA-P}} & \ad{{MoCoDA}} (not factored) & \ad{{MoCoDA-P}} (not factored) \\
\midrule
\env{BCQ} & 2.0 $\pm$ 1.6 &  20.7 $\pm$ 4.1 &  \textbf{64.7 $\pm$ 4.1} & 14.0 $\pm$ 3.3 & 15.3 $\pm$ 4.1 \\
\env{TD3-BC} & 0.7 $\pm$ 0.9 &  38.7 $\pm$ 7.5 &  \textbf{84.0 $\pm$ 2.8} & 29.3 $\pm$ 3.8 & 26.0 $\pm$ 1.6 \\
	\bottomrule
\end{tabularx}
	\vspace{0.02in}
	\caption{\textbf{HookSweep2 Offline RL Results:} Average success percentage ($\pm$ std. dev. over 3 seeds), where the average was computed over the last 50 training epochs. 
	\env{SAC} and \env{CQL} (omitted) were unsuccessful with all datasets.
	We see that \ad{Mocoda} was necessary for learning, and that results improve drastically with \ad{Mocoda-P}, which re-balances \ad{Mocoda} toward a uniform distribution in the box coordinates (see Figure \ref{fig_fetch_parent_distributions}).
	Additionally, we show results from an ablation, which generates the MoCoDA datasets using a fully connected dynamics model. While this still achieves some success, it demonstrates that using a locally-factored model is important for OOD generalization. In this case the more OOD MoCoDA-P distribution does not help, suggesting that the fully connected model is failing to produce useful OOD transitions.
	}
	\label{table_fetch_batchrl_results}
	\vspace{\baselineskip}
\end{table*}
}

\section{Conclusion}\label{sec:conclusion}
In this paper, we tackled the challenging yet common setting where the available empirical data provides insufficient coverage of critical parts of the state space. 
Starting with the insight that locally factored transition models are capable of generalizing outside of the empirical distribution, we proposed \methodName, a framework for augmenting available data using a controllable ``parent distribution'' and locally factored dynamics model.
We find that adding augmented samples from \methodName allows RL agents to learn policies that traverse states and actions never before seen in the experience buffer.
Although our data augmentation is ``model-based'', the transition samples it produces are compatible with any downstream RL algorithm that consumes single-step transitions. 

Future work might (1) explore methods for learning locally factorized representations, especially in environments with high-dimensional inputs (e.g., pixels) \cite{jiang2019scalor,kipf2019contrastive}, and consider how \methodName might integrate with latent representations, (2) combine the insights presented here with learned predictors of out-of-distribution generalization (e.g., uncertainty-based prediction) \cite{pan2020trust}, (3)
create benchmark environments for entity-based RL \cite{winter2022} so that object-oriented methods and models can be better evaluated, and (4) explore different approaches to re-balancing the training distribution for learning on downstream tasks. With regards to direction (1), we note that asserting (or not) certain independence relationships may have fairness implications for datasets \cite{park2018reducing,creager20a} that should be kept in mind or explored. This is relevant also in regards to direction 4, as dataset re-balancing may result in (or fix) biases in the data \cite{krasanakis2018adaptive}. Re-balancing schemes should be sensitive to this. 

\begin{ack}
We thank Jimmy Ba, Marc-Etienne Brunet, and Harris Chan for helpful comments and discussions. We also thank the anonymous reviewers for their feedback, which significantly improved the final manuscript.
Silviu Pitis is supported by an NSERC CGS-D award. Animesh Garg is supported as a CIFAR AI chair, and by an NSERC Discovery Award, University of Toronto XSeed Grant and NSERC Exploration grant.
Resources used in preparing this research were provided, in part, by the Province of Ontario, the Government of Canada, and companies sponsoring the Vector Institute.\end{ack}

\vfill

\newpage
{\small
\bibliographystyle{plainnat}
\bibliography{refs}
}

\newpage

\section*{Checklist}


\begin{enumerate}

\item For all authors...
\begin{enumerate}
  \item Do the main claims made in the abstract and introduction accurately reflect the paper's contributions and scope?
    \answerYes{}
  \item Did you describe the limitations of your work?
    \answerYes{See \autoref{subsection_generating_ad}}
  \item Did you discuss any potential negative societal impacts of your work?
    \answerYes{See \autoref{sec:conclusion} and \autoref{broader_impacts}}
  \item Have you read the ethics review guidelines and ensured that your paper conforms to them?
    \answerYes{}
\end{enumerate}

\item If you are including theoretical results...
\begin{enumerate}
  \item Did you state the full set of assumptions of all theoretical results?
    \answerYes{See \autoref{appdx_proposition}}
  \item Did you include complete proofs of all theoretical results?
    \answerYes{See \autoref{appdx_proposition}}
\end{enumerate}

\item If you ran experiments...
\begin{enumerate}
  \item Did you include the code, data, and instructions needed to reproduce the main experimental results (either in the supplemental material or as a URL)?
    \answerYes{https://github.com/spitis/mocoda}
  \item Did you specify all the training details (e.g., data splits, hyperparameters, how they were chosen)?
    \answerYes{Yes in \autoref{sec:empirical} and in the \autoref{exp_details}}
        \item Did you report error bars (e.g., with respect to the random seed after running experiments multiple times)?
    \answerYes{}{}
        \item Did you include the total amount of compute and the type of resources used (e.g., type of GPUs, internal cluster, or cloud provider)?
    \answerYes{Yes in \autoref{exp_details}}
\end{enumerate}

\item If you are using existing assets (e.g., code, data, models) or curating/releasing new assets...
\begin{enumerate}
  \item If your work uses existing assets, did you cite the creators?
    \answerYes{We cited \cite{mrl,mandlekar2020learning} for our RL framework \& Hook environment assets; both are open-source.}
  \item Did you mention the license of the assets?
    \answerYes{}
  \item Did you include any new assets either in the supplemental material or as a URL?
    \answerYes{The changes will be included with the code.}
  \item Did you discuss whether and how consent was obtained from people whose data you're using/curating?
    \answerNA{}
  \item Did you discuss whether the data you are using/curating contains personally identifiable information or offensive content?
    \answerNA{}
\end{enumerate}

\item If you used crowdsourcing or conducted research with human subjects...
\begin{enumerate}
  \item Did you include the full text of instructions given to participants and screenshots, if applicable?
    \answerNA{}
  \item Did you describe any potential participant risks, with links to Institutional Review Board (IRB) approvals, if applicable?
    \answerNA{}
  \item Did you include the estimated hourly wage paid to participants and the total amount spent on participant compensation?
    \answerNA{}
\end{enumerate}

\end{enumerate}


\newpage
\appendix

\section{Proof of Theorem 1}\label{appdx_proposition}
The dynamics model assumed is a maximum-likelihood, count-based model that has separate parameters for each causal mechanism, $P_{i, \theta}^\L$, in each local neighborhood.
That is, for a given configuration of the parents $\parents_i = x$ in $P_{i, \theta}^\L$, we define count parameter $\theta_{ij}$ for the $j$-th possible child, $c_{ij}$, so that $P_{i, \theta}^\L(c_{ij}\,|\,x) = \theta_j/\sum_{k=1}^{|c_i|}\theta_k$.

We use the following two lemmas (see source material for proof):
\vspace{\baselineskip}

\begin{lemma}[Proposition A.8 of \citet{agarwal2019reinforcement}]\label{lemma_concentration_l1}
Let $z$ be a discrete random variable that takes values in $\{1, \dots, d\}$, distributed according to $q$.
We write $q$ as a vector where $\vec{q} = [\textrm{Pr}(z=j)]_{j=1}^d$.
Assume we have $n$ i.i.d. samples, and that our empirical estimate of $\vec{q}$ is $[\vec{q}]_j = \sum_{i=1}^n\mathbf{1}[z_i=j]/n$.\\
We have that $\forall\epsilon > 0$:
$$\textrm{Pr}(\Vert\hat{q} - \vec{q}\Vert_2 \geq 1 / \sqrt{n} + \epsilon) \leq e^{-n\epsilon^2}$$
which implies that:
$$\textrm{Pr}(\Vert\hat{q} - \vec{q}\Vert_1 \geq \sqrt{d}(1/\sqrt{n} + \epsilon)) \leq e^{-n\epsilon^2}$$

\end{lemma}
\vspace{\baselineskip}

\begin{lemma}[Corollary 1 of \citet{strehl2007model}]\label{lemma_aggregate_l1}
If for all states and actions, each model $P_{i,\theta}$ of $P_i$ is $\epsilon/k$ close to the ground truth in terms of the $\lone$ norm: $\Vert P_i(s, a) - P_{i,\theta}(s, a) \Vert_1 < \epsilon/k$, then the aggregate transition model $P_\theta$ is $\epsilon$ close to the ground truth transition model: $\Vert P(s, a) - P_{\theta}(s, a) \Vert_1 < \epsilon$.
\end{lemma}
\vspace{\baselineskip}

\begin{appdxTheorem}{1}
Let $n$ be the number of empirical samples used to train the model of each local causal mechanism $P_{i, \theta}^\L$ at each configuration of parents $\parents_i\! =\! x$.
There exists positive constant $c$ such that, if
$$
n \geq \frac{ck^2|c_i|\log(|\S||\A|/\delta)}{\epsilon^2},
$$
then, with probability at least $1-\delta$, we have:
$$\max_{(s, a)} \Vert P(s, a) - P_{\theta}(s, a) \Vert_1 \leq  \epsilon.$$
\end{appdxTheorem}

\begin{proof}
Applying Lemma \ref{lemma_concentration_l1}, we have that for fixed parents $\parents_i = x$, wp. at least $1-\delta$,

$$\Vert P_i(x) - P_{i,\theta}(x) \Vert_1 \leq c\sqrt{\frac{|c_i|\log(1/\delta)}{n}},$$

where $n$ is the number of independent samples used to train $P_{i,\theta}$ and $c$ is a positive constant.
Now consider a fixed $(s,a)$, consisting of $k$ parent sets.
Applying Lemma \ref{lemma_aggregate_l1} we have that, wp. at least $1\!-\!\delta$, 

$$\Vert P(s,a) - P_{\theta}(s,a) \Vert_1 \leq ck\sqrt{\frac{|c_i|\log(1/\delta)}{n}}.$$

We apply the union bound across all states and actions to get that wp. at least $1-\delta$,

$$\max_{(s, a)} \Vert P(s, a) - P_{\theta}(s, a) \Vert_1 \leq ck\sqrt{\frac{|c_i|\log(|S||A|/ \delta)}{n}}.$$

The result follows by rearranging for $n$ and relabeling $c$.
\end{proof}

To compare to full-state dynamics modeling, we can translate the sample complexity from the per-parent count $n$ to a total count $N$.
Recall $m\Pi_i|c_i| = |\S|$, so that $|c_i| = (|\S|/m)^{1/k}$, and $m\Pi_i|\parents_i| \geq |\S||\A|$.
We assume a small constant overlap factor $v \geq 1$, so that $|\parents_i| = v(|\S||\A|/m)^{1/k}$.
We need the total number of component visits to be $n|\parents_i|km$, for a total of $nv(|\S||\A|/m)^{1/k}m$ state-action visits, assuming that parent set visits are allocated evenly, and noting that each state-action visit provides $k$ parent set visits.
This gives:

\begin{appdxCorollary}{1}
To bound the error as above, we need to have
$$N \geq \frac{cmk^2(|\S|^2|\A|/m^2)^{1/k}\log(|\S||\A|/\delta)}{\epsilon^2},$$
total train samples, where we have absorbed the overlap factor $v$ into constant $c$.
\end{appdxCorollary}

To extend this and adapt other results to our setting, we could now apply the Simulation Lemma \cite{agarwal2019reinforcement} to bound the value difference given the model error, or alternatively, develop the theory in the direction of \cite{strehl2007model} and related work.
However, we believe the core insights are already contained in Theorem 1 and Corollary 1. 

\section{Implementation Details}\label{impl_details}

Code is available at https://github.com/spitis/mocoda (using https://github.com/spitis/mrl for RL algorithms).
There are numerous components involved that each have several different settings that were mostly just taken ``as-is'' or picked as reasonable defaults (e.g., using a layer size of 512 in most neural networks, or having 5 components in the MDN, or the specific implementation of rejection sampling for \texttt{Mocoda-U}).
The best documentation for specific details is the code itself.
As such, the implementation details below cover the broad strokes so that a reader might understand the general pipeline, and we refer the reader to the provided code for precise details.

\subsection{Causal Transition Structure and Parent Set Definitions}

We implement the local causal model as a mask function $M$ that maps (state, action) tuples to an adjacency matrix of the causal structure.
For example, in \texttt{2d Navigation}, the mask function was implemented as follows:

{\footnotesize
\begin{verbatim}
  def Mask2dNavigation(input_tensor):
    """
    accepts B x num_sa_features, and returns B x num_parents x num_children
    """

    # base local mask
    mask = torch.tensor(
      [[1, 0],
       [0, 1],
       [1, 0],
       [0, 1]]).to(input_tensor.device)

    # change local mask in top right quadrant
    mask = mask[None].repeat((input_tensor.shape[0], 1, 1))
    mask[torch.logical_and(input_tensor[:,0] > 0.5, input_tensor[:, 1] > 0.5)] = 1
    
    return mask
\end{verbatim}
}
As an example, the causal graph for the base local mask, which applies for most of the state space is shown in the figure ~\ref{fig:cg-local-mask}.
We used the base local graph to select the parent sets, in this case, $(x, \Delta x)$ and $(y, \Delta y)$. 

\begin{figure}[!h]
	\centering
	\includegraphics[width=0.2\textwidth]{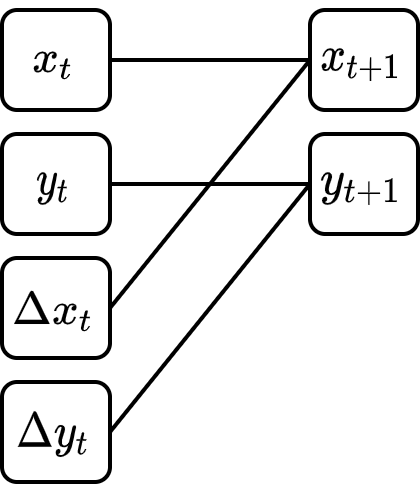}
	\caption{Causal graph for local mask}
	\label{fig:cg-local-mask}
\end{figure}

\subsection{Parent Distribution}\label{appdx_parentdist}

To sample the parent distribution in Step 1 of \methodName, we use the Gaussian Mixture Model (GMM) based approach described in the main text.
The advantage of this approach is that we can easily do conditional sampling in case of overlapping parent sets.
For a given local subset $\L$, we fit a separate GMM to the marginal of each parent set, as it appears in the empirical distribution for $\L$.
To generate a new sample, we optionally shuffle the GMMs, and then sample from one GMM at a time, conditioning on any already generated features.
This process eliminates any spurious correlations between features that are not part of the same parent set, and thus results in the maximum-entropy, marginal matching distribution.

\begin{figure}[!t]
	\centering
	\includegraphics[width=0.8\textwidth]{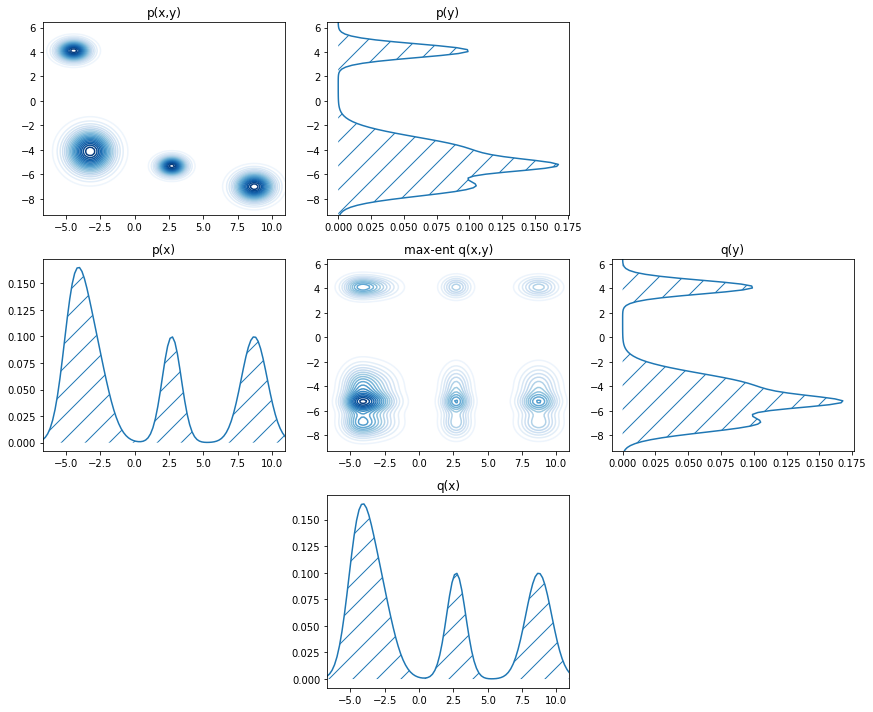}
	\caption{Hypothetical 2D illustration of the GMM-based parent set sampler.
  It is assumed that there are two non-overlapping parent sets $\{x\}$ and $\{y\}$, but that $x$ and $y$ exhibit dependence in the empirical data.
  We fit a GMM to each marginal $P(x)$ and $P(y)$ and sample from them independently to get $Q(x, y)$, which has the same marginal distributions (so that the components in a locally factored dynamics model will generalize), but eliminates the spurious dependence in the empirical data.}
\end{figure}

In cases of multiple local neighborhoods, $\L_1, \L_2, \dots$, one should respect the boundaries of the current local subset $\L$ during both training \textit{and} generation.
If a sample generated with the GMM for $\L$ falls outside of $\L$, that sample should be rejected, as the local causal structure is no longer valid, and the generalization guarantee for the locally factored model no longer holds.

As the local factorization in our experiments is quite simple, we did not stratify the GMM generator, and instead used a single GMM generator for the sparsest local causal structure.
In the case of \texttt{2d Navigation} this did not generate any data that was out-of-distribution for the locally factored model components (as the agent's policy was consistent in all local neighborhoods).
In the case of \texttt{HookSweep2}, there was a bit of locally out-of-distribution data in the local subspace in which there is a block collision; however, most of this data is unreachable as it involves overlapping blocks, and we obtained strong results even with this shortcut.

\subsection{Dynamics Models}\label{sec:dynamics-models}

Our experiments used three different dynamics models.
In each case, we used an ensemble of 5 base models, described below.
The base models output a Gaussian mean and variance for each output variable and are trained independently via a negative log likelihood loss.
All models are trained using Adam Optimizer~\cite{kingma2014adam}.

\begin{enumerate}
    \item \textbf{Unfactored:} The base model is a fully connected neural network with ReLU activations (MLP).
    \item \textbf{Globally Factored:} The base model has one MLP for each causal mechanism in the sparsest local graph.
    For both \texttt{2d Navigation} and \texttt{HookSweep2} the sparest local graph has two components, so the base global model is composed of two MLPs.
    \item \textbf{Locally factored:} The base model is designed as follows.
    For each child node, $c_i$, there is a separately parameterized single MLP that is preceded by a ``Masked Composer'' module.
    The Masked Composer applies a single layer MLP (linear transform followed by ReLU) to each root node, $r_i$ (each parent set has several nodes), to obtain embeddings $\varepsilon_i(r_i)$.
    The $i$-th column of the mask is used to zero out the corresponding embeddings which are then summed, $\sum_i M_{ij}\varepsilon_i(r_i)$, and the result is passed as an input to the MLP.
    
    This architecture works (and enforces local factorization), but is likely poor, because it does not take advantage of potentially useful shared representations between parent nodes across children (since there is a separately parameterized Masked Composer for each child).
    A better architecture would likely use a single parameterization for a single, possibly deeper Masked Composer.
    As this is not the focus of our contribution, we stuck with simple model, as it ``just worked'' for purposes of our experiments. 
\end{enumerate}

\subsection{Training Data for the RL Algorithm}

This varied by experiment, and is described in the next Section.
Notably, we divided the standard deviation returned by our dynamics models by a factor of three when generating data to avoid data that was too far out of distribution.

\subsection{Reinforcement Learning Algorithms}

We use Modular RL \cite{mrl}, adding three offline RL \cite{levine2020offline} algorithms: 
{BCQ} \citep{fujimoto2019off}, {CQL} \citep{kumar2020conservative} \& {TD3-BC} \citep{fujimoto2021minimalist}.

The {BCQ} implementation uses DDPG \cite{lillicrap2015continuous}.
For the generative model we use a Mixture Density Network (MDN) \cite{bishop1994mixture} with 5 components, that produces 20 action samples at each call (both during test rollouts and when creating critic targets).
The MDN was trained for 1000 batches with batch size of 2000.
We did not use a perturbation model. 

The {CQL} implementation uses SAC \cite{haarnoja2018soft}.
Rewards in our environments are sparse, and so value targets can be accurately clipped between two values (depends on the discount factor).
CQL balances two losses: a penalty for Q-values of some non-behavioral distribution/policy (we use a random policy), and a bonus for the Q-values behavioral actions.
We use an L1 penalty toward the lower end of the value target clipping range, and an L1 bonus toward the higher end of the value target clipping range.
We then multiply that by a minimum Q coefficient, as in the original CQL implementation. 

The TD3-BC implementation follows \citet{fujimoto2021minimalist}. 

\section{Experimental Details}\label{exp_details}

\subsection{2D Navigation}  In this environment, the agent must travel from one point in a square arena to another.
States are 2D $(x, y)$ coordinates and actions are 2D $(\Delta x, \Delta y)$ vectors. 

{\footnotesize
\begin{verbatim}
    observation_space = spaces.Box(np.zeros((2,)), np.ones((2,)), dtype=np.float32)
    action_space = spaces.Box(-np.ones((2,)), np.ones((2,)), dtype=np.float32)
\end{verbatim}
}

Episodes run for up to 70 steps.
Rewards are sparse, with a -1 reward everywhere except the goal, where reward is 0.
In most of the state space, the sub-actions $\Delta x$ and $\Delta y$ affect only their respective coordinate.
In the top right quadrant, however, the $\Delta x$ and $\Delta y$ sub-actions each affect \textit{both} $x$ and $y$ coordinates, so that the environment is locally factored.
The two causal graphs are as follows:

\newpage 

\begin{figure}[!h]
	\centering
	\includegraphics[width=0.45\textwidth]{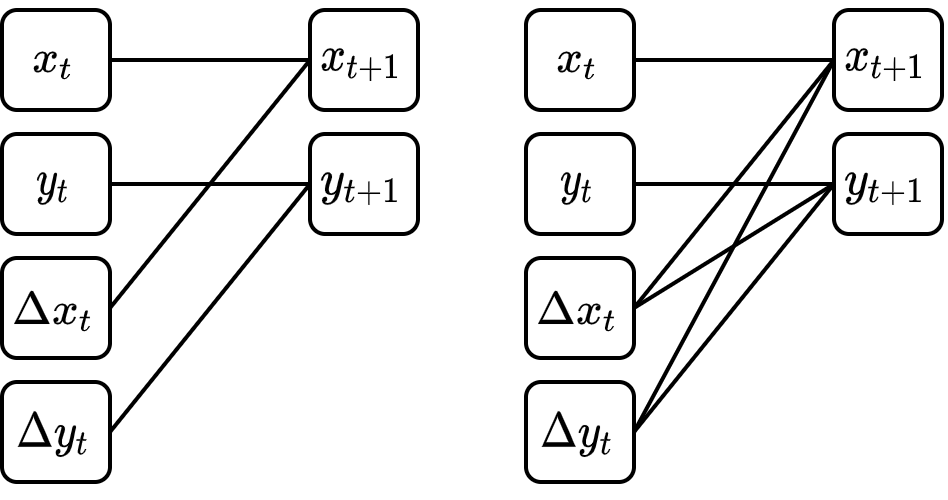}
\end{figure}

The graph on the right applies only in the top-right quadrant; otherwise the graph on the left applies.
The graph on the left has non-overlapping parent sets $(x, \Delta x)$ and $(y, \Delta y)$.
The graph on the right has overlapping parent sets $(x, \Delta x, \Delta y)$ and $(y, \Delta x, \Delta y)$.

The agent has access to an empirical dataset consisting of left-to-right \& bottom-to-top trajectories (20,000 transitions of each type):

\begin{figure}[!h]
	\centering
	\includegraphics[width=0.8\textwidth]{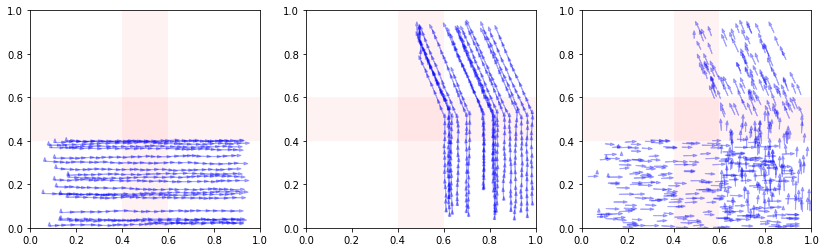}
	\caption{\textbf{Left/Middle:} Random samples of the two types of trajectories the agent has access to. \textbf{Right:} Random sample of transitions from this empirical dataset.}
\end{figure}

We consider a target task where the agent must move from the bottom left to the top right.
In this task there is sufficient empirical data to solve the task by following the ${\Large \textbf{$\lrcorner$} }$ shape of the data, but learning the optimal policy of going directly via the diagonal requires out-of-distribution generalization.

For \env{2d Navigation}, we generated the \ad{Mocoda} distribution by fitting a GMM generator as described in the previous Section.
Each GMM (one for each parent set) had 32 components, and was fit using expectation maximization.
To obtain \ad{Mocoda-U}, we implemented rejection sampling by using a KDE density estimator is as follows:

{\footnotesize
\begin{verbatim}
def prune_to_uniform(proposals, target_size=12000.):
  from sklearn.neighbors import KernelDensity
  sample = proposals[-10000:]

  fmap = lambda s: s[:, :2]
  K = KernelDensity(bandwidth=0.05)
  K.fit(fmap(sample))
  scores = K.score_samples(fmap(proposals))
  scores = np.maximum(scores, np.log(0.01))
  scores = (1. / np.exp(scores))
  scores = scores / scores.mean()  * (target_size / len(proposals))
  
  return proposals[np.random.uniform(size=scores.shape) < scores]   
\end{verbatim}
}

The dynamics models each had 2 layers of 256 neurons and were trained with a batch size of 512 and learning rate of 1e-4.
Hyperparameters were not tuned once a working setting was found.
Of the 40K empirical samples, 35K were used for training, and 5000 for validation.
The models were trained for 600 epochs, with early stopping used in the last 50 epochs to find a locally optimal stopping point. 

Augmented datasets of 200K samples were generated.
In each case except \ad{Emp}, 40K were the original empirical dataset (thus 160K new samples were generated by applying the dynamics model to samples from the augmented distribution).
In case of \ad{Emp}, the 40K original samples were simply repeated 5 times to get a size 200K dataset.
The locally factored network was used to generate the augmented datasets. 

These augmented distributions were then used to train the downstream RL agents.
The agent algorithms used a discount factor of 0.98, a target cutoff range of (-50, 0), batch size of 500, and used 2 layers of 512 neurons in both actor and critic networks.
The agents were trained for 25K batches (for a total of 62.5 passes over the dataset). 

For \env{2D Navigation} we ran 5 seeds, which all yielded similar results.
For each seed we trained new parent set samplers and generated new augmented datasets.

\subsection{HookSweep2}

\env{HookSweep2} is a challenging robotics domain based on Hook-Sweep~\citep{kurenkov2020ac}, in which a Fetch robot must use a long hook to sweep two boxes to one side of the table (either toward or away from the agent).
States, excluding the goal, are 16 dimensional continuous vectors.
Goals are 6 dimensions.
The agents all concatenate the goal to the state, and so operate on 22 dimensional states.
The action space is a 4 dimensional continuous vector. 

The environment contains two boxes that are initialized near the center of the table.

The empirical data contains 1M transitions from trajectories of an expert agent sweeping exactly one box to one side of the table, leaving the other in the center. The target task requires the agent to sweep \textit{both} boxes together to one side of the table.
This is particularly challenging because the setup is entirely offline (no exploration), where poor out-of-distribution generalization typically requires special offline RL algorithms that constrain the agent's policy to the empirical distribution~\citep{levine2020offline,agarwal2020optimistic,kumar2020conservative,fujimoto2021minimalist}.

Episodes run for 75 steps.
Rewards are dense, but structured similarly to a sparse reward, with a base reward of -1 everywhere except the goal and a reward of 0 at the goal.
Additional small rewards are given if the agent keeps the hook near the table (this was required to obtain natural movements from the trained expert agent).

In this environment, \textit{we did not have the ground truth causal graph}, and so a heuristic was used.
The heuristic (wrongly) assumes that the agent/hook \textit{always} causes each of the next object position (hook and objects are always entangled), even though this is only true when the hook and the objects are touching.
The heuristic considers the two boxes to be separate whenever they are further than 5cm from each other.
Here is the implementation of the heuristic:

{\scriptsize
\begin{verbatim}
  def MaskHookSweep2(input_tensor):
  
    # base local mask for when boxes are far apart
    mask = torch.tensor(
      [[1, 1, 1],
      [1, 1, 0],
      [1, 0, 1],
      [1, 1, 1]]
      ).to(input_tensor.device)
    mask = mask[None].repeat((input_tensor.shape[0], 1, 1))
    
    # change local mask when boxes are close to each other
    mask[torch.sum(torch.abs(input_tensor[:,O1X:O1X+2] -\ 
          input_tensor[:,O2X:O2X+2]), axis=1) < 0.05] = 1
    
    return mask
\end{verbatim}
}

where the state-action components are (gripper, box1, box2, action).
This heuristic returns the following two causal graphs (note that goals are not part of the dynamics, and are separately labeled using random goal samples from the environment):

\begin{figure}[!h]
	\centering
	\includegraphics[width=0.4\textwidth]{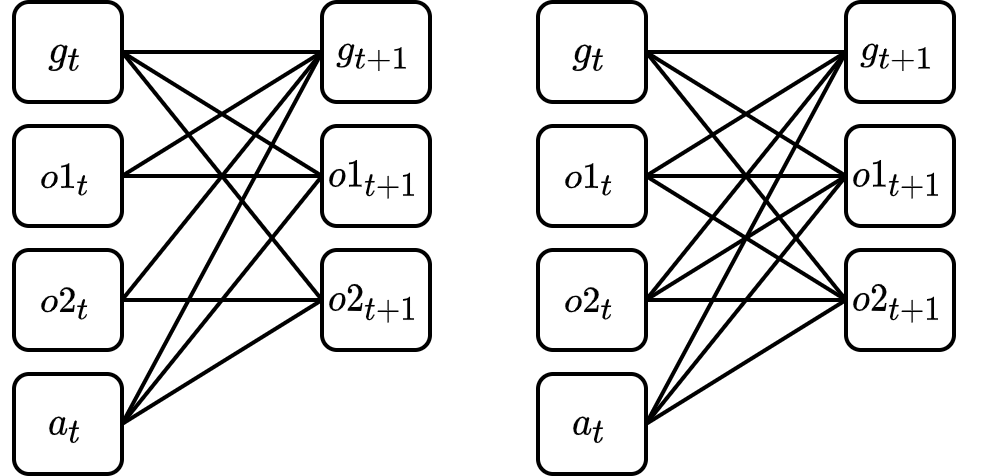}
\end{figure}

The parent sets are (g, o1, a) and (g, o2, a) for the first graph, and (g, o1, o2, a) in the second graph. 

For \env{HookSweep2}, the generation of the \ad{Mocoda} distribution is identical to how it was generated in \env{2d Navigation} (see previous subsection).
To obtain \ad{Mocoda-P}, we implemented rejection sampling as follows:

{\scriptsize
\begin{verbatim}
  def prune_to_uniform2(proposals, target_size=12000., smaller=True):
    proposals = proposals[np.linalg.norm(proposals[:,O1X:O1X+2] - proposals[:,O2X:O2X+2], axis=-1) < 0.3]
    sample = proposals[-5000:]
    
    fmap = lambda s: s[:,[O1X,O1X+1,O2X,O2X+1]]
    K = KernelDensity(bandwidth=0.05)
    K.fit(fmap(sample))
    scores = K.score_samples(fmap(proposals))
    scores = np.maximum(scores, np.log(0.05))
    scores = (1. / np.exp(scores))
    if np.minimum(scores, 1).sum() > 10000:
    while np.minimum(scores, 1).sum() > 10000:
      scores = scores * 0.99
    else:
    while np.minimum(scores, 1).sum() < 10000:
      scores = scores / 0.99
    
    return proposals[np.random.uniform(size=scores.shape) < scores]   
\end{verbatim}
}

The key difference to the \env{2d Navigation} is the definition of the \texttt{fmap} function, which defines the feature map under which the density is computed for rejection sampling.

The dynamics models for \env{HookSweep2} each had 2 layers of 512 neurons and were trained with a batch size of 512 and learning rate of 2e-4.
Hyperparamters were not tuned once a working setting was found (learning rate was increased to make training slightly faster).
Of 1M empirical samples, 5000 were used for validation.
The models were trained for 4000 epochs, where each epoch involved 40K random samples, with early stopping used in the last 50 epochs to find a locally optimal stopping point.

Augmented datasets of 5M samples were generated.
In each case except \ad{Emp}, 1M were the original empirical dataset (thus 4M new samples were generated).
In the case of \ad{Emp}, the 1M original samplers were simply repeated 5 times to get the full augmented dataset.
The locally factored network was used to generate the augmented datasets. 

These augmented distributions were then used to train the downstream RL agents.
The agent algorithms were the same as for \env{2d Navigation}, except that they used 3 layers of 512 neurons in both actor and critic networks.
The agents were trained for 1M steps with batch size 500 (for a total of 100 passes over the dataset). 

\subsection{Licenses and Compute}

All experiments were run on a modern desktop CPU and a NVIDIA GTX 1080 Ti GPU. 

Code and assets are available under Apache and MIT licenses from Mujoco, OpenAI Gym, AC-Teach~\cite{kurenkov2020ac}, and Modular RL~\cite{mrl} repositories.
The implementations used in this paper will be released upon acceptance under an open source license. 

\section{Further Discussion of Broader Impacts}\label{broader_impacts}

\methodName uses causally-motivated data augmentation to tackle sequential decision making problems where the available experience data may not be sufficient to find an optimal policy for the task at hand.
While we have thusfar applied this approach to continuous control problems, there are a large body of problems that share this general motivation, where long-term fairness and robustness may be a central concern \citep{d2020fairness}.
In these cases, the causal assumptions used to implement \methodName deserve extra care and external scrutiny.
For such problems, the structure of the state space may include sensitive and/or socially-ascribed attributes of groups and individuals (which cannot be directly intervened upon), so any graphical causal will involve normative assumptions about the environment in which the agent is embedded \citep{hu2020s}.

\section{\methodName sampling pseudocode}

%
\algnewcommand{\LeftComment}[1]{\Statex \(\triangleright\) #1}
\newcommand{\Pai}{\text{Pa}(i)}
\newcommand{\already}{\mathcal{AS}}
\begin{algorithm}\footnotesize
  \centering
  \begin{minipage}{1.\linewidth}
	\caption{\methodName for FMDPs, assuming hand-specified dynamics factorization}
  \label{algo:mocoda-fmdp}
	\begin{algorithmic}[1]
	\Function{GenerateMocodaData}{N}:
	\State \textbf{Input:} observed transition dataset $(s,a,s') \in \mathcal{D}$
	\State \textbf{Input:} causal structure of transition dynamics $\mathcal{G} := \{(i, \Pai) \ \forall \ i \in N_s\}$
	a.k.a. ``parent sets''
	\State \hphantom{\textbf{Input:}} NOTE: $\text{Pa}(i)$ is shorthand for $\{j : s_j \in \text{CausalParent}(s_i')\}$
	\State \hphantom{\textbf{Input:}} NOTE: $\text{Pa}(i) \subset [N_s + N_a]$ can index into states or actions
	\vspace{0.5\baselineskip}
	\State $\mathcal{D}_{tr}, \mathcal{D}_{va} = \texttt{train\_val\_split}(\mathcal{D})$
	\Comment{Split data}
  \State $\theta := \textsc{TrainGMMParentsModel}(\mathcal{D}_{tr}, \mathcal{D}_{va}, \mathcal{G})$
	\Comment{train GMM parent distribution $P_\theta(s,a)$}
  
  \State $\phi := \textsc{TrainFactoredDynamics}(\mathcal{D}_{tr}, \mathcal{D}_{va}, \mathcal{G})$
	\Comment{train factored dynamics model $P_\phi(s'|s,a)$}
	
  \State \textbf{return} \textsc{SampleAugmentedDataset}(N, $\theta$, $\phi$, $\mathcal{G}$)
  \EndFunction
  
  \vspace{\baselineskip}
  \Function{TrainGMMParentsModel}{$\mathcal{D}_{tr}, \mathcal{D}_{va}, \mathcal{G}$}:
    \For{ $(i, \Pai) \in \mathcal{G}:$ } \Comment iterate over parent sets for each child
      \State $\{\mu^k_{\Pai},\Sigma^k_{\Pai},\gamma^k_{\Pai}\} = \texttt{init\_gmm\_params}(N_k)$
      \Comment{NOTE: \emph{only} for this parent set}
      \State $\mathcal{D}^{tr}_{\Pai}(s,a) := \{(s[\Pai], a[\Pai \% N_s]) \forall (s, a, s') \in \mathcal{D}^{tr}\}$
      \State $\mathcal{D}^{va}_{\Pai}(s,a) := \{(s[\Pai], a[\Pai \% N_s]) \forall (s, a, s') \in \mathcal{D}^{va}\}$ \Comment subsample relevant dims
      \While{$\texttt{not\_converged}(\text{GMM}(\cdot;\mu^k_{\Pai},\Sigma^k_{\Pai},\gamma^k_{\Pai}), \mathcal{D}^{va}_{\Pai}(s,a))$}
          \State $\{\mu^k_{\Pai},\Sigma^k_{\Pai},\gamma^k_{\Pai}\} \leftarrow \texttt{update\_gmm\_params}(\{\mu^k_{\Pai},\Sigma^k_{\Pai},\gamma^k_{\Pai}\}, \mathcal{D}^{tr}_{\Pai}(s,a))$
      \EndWhile
      \State $\theta \texttt{.append}(\{\mu^k_{\Pai},\Sigma^k_{\Pai},\gamma^k_{\Pai}\})$
  \EndFor
  \State \textbf{return} $\theta$ 
  \EndFunction
  
  \vspace{\baselineskip}
  \Function{TrainFactoredDynamics}{$\mathcal{D}_{tr}, \mathcal{D}_{va}, \mathcal{G}$}:
  \For{ $(i, \Pai) \in \mathcal{G}:$ } \Comment iterate over parent sets for each child
      \State $\phi_i = \texttt{init\_mlp\_params}()$ 
      \Comment{NOTE: \emph{only} for this parent set}
      \State $\mathcal{D}^{tr}_{\Pai}(s,a,s') := \{(s[\Pai], a[\Pai \% N_s], s'[i]) \forall (s, a, s') \in \mathcal{D}^{tr}\}$
      \State $\mathcal{D}^{va}_{\Pai}(s,a,s') := \{(s[\Pai], a[\Pai \% N_s], s'[i]) \forall (s, a, s') \in \mathcal{D}^{va}\}$ \Comment subsample relevant dims
      \While{$\texttt{not\_converged}(\text{MLP}(\cdot|\cdot;\phi_i), \mathcal{D}^{va}_{\Pai}(s,a,s'))$}
          \State $\phi_i \leftarrow \texttt{update\_mlp\_params}(\phi_i, \mathcal{D}^{tr}_{\Pai}(s,a,s'))$
      \EndWhile
  \EndFor
  \State \textbf{return} $\phi$ 
  \EndFunction
  
  \vspace{\baselineskip}
  \Function{SampleAugmentedData}{N, $\theta$, $\phi$, $\mathcal{G}$}
    \For{$\_ \in \texttt{range}(N)$:}
	    \LeftComment{\ \ \ sample parent data, i.e. $(\tilde s, \tilde a)$: sequentially sample the parent set GMMs, conditioning}
      \LeftComment{\ \ \ each GMM on previous samples to handle any overlap between parent sets}
      \State $\already := \{ \ \}$
      \Comment{define an ``already sampled'' set to track any parent set overlap}
      \State $\tilde s := [ \ ]$;
             $\tilde a := [ \ ]$
      \For{$i \in \texttt{range}(N_s)$:}
         \If{$\Pai \cap \already = \emptyset:$} \Comment{no vars in this parent set already sampled}
            \State $(\tilde s_{\Pai}, \tilde a_{\Pai}) \sim \text{GMM}(\cdot;\mu^k_{\Pai},\Sigma^k_{\Pai},\gamma^k_{\Pai})$
            \State $\tilde s \texttt{.extend}(\tilde s_{\Pai})$
            \State $\tilde a \texttt{.extend}(\tilde a_{\Pai})$
         \Else \Comment{some vars in this parent set already sampled and must be conditioned on}
	          \LeftComment{\ \ \ \ \ \ \ \ \ \ \ \ \ \ \ \ condition this GMM already-sampled vars, then sample remaining vars}
            \State $(\tilde s_{\Pai \backslash \already}, \tilde a_{\Pai \backslash \already}) \sim \text{GMM}(\cdot|\already;\mu^k_{\Pai},\Sigma^k_{\Pai},\gamma^k_{\Pai})$
	          \LeftComment{\ \ \ \ \ \ \ \ \ \ \ \ \ \ \ \ NOTE: this sampling is easily realized by conditioning each Gaussian component}
            \LeftComment{\ \ \ \ \ \ \ \ \ \ \ \ \ \ \ \ and updating mixture components in proportion to density of already-sampled vars}
            \State $\tilde s \texttt{.extend}(\tilde s_{\Pai \backslash \already})$
            \State $\tilde a \texttt{.extend}(\tilde a_{\Pai \backslash \already})$
         \EndIf
         \State $\already \leftarrow \already \cup \Pai$
      \EndFor
	    \LeftComment{\ \ \ sample next states, i.e. $\tilde s'|(\tilde s, \tilde a)$: sequentially sample each ``factor'' in the factorized dynamics}
      \State $\tilde s' := [ \ ]$
      \For{$i \in \texttt{range}(N_s)$:}
          \State $\tilde s_i' \sim \text{MLP}(\cdot|\tilde s, \tilde a;\phi_i)$
          \State $\tilde s' \texttt{.append}(\tilde s_i')$
      \EndFor
      \LeftComment{\ \ \ assemble transition}
      \State $\tilde s  = \texttt{array}(\tilde s )$;
             $\tilde a  = \texttt{array}(\tilde a )$;
             $\tilde s' = \texttt{array}(\tilde s')$
      \State $\mathcal{\tilde D}\texttt{.append}((\tilde s, \tilde a, \tilde s'))$
  \EndFor
  \State \textbf{return} $\mathcal{\tilde D}$
  \EndFunction
	\end{algorithmic} 
  \end{minipage}
\end{algorithm}

Algorithm \ref{algo:mocoda-fmdp} shows the pseudocode for \methodName sampling for a FMDP, 
where the causal structure of the transition dynamics is assumed known.
For simplicity of exposition we describe the case where each ``factor'' in the 
factorized dynamics is modeled using an MLP, which corresponds to the ``Globally Factored''
model architecture referred to in Table \ref{tab:toy-mse}.
Realizing the ``Locally Factored'' architecture is simply a matter of replacing $\text{MLP}(\cdot)$
in the pseudocode with $(\text{MLP} \circ \text{MaskedComposer})(\cdot)$ described in Section \ref{sec:dynamics-models}.

Implementing \methodName sampling for a Local Causal Model rather than an FMDP is also a straightforward extension. The dynamics modeling is the same (but for the dynamics being conditioned on $\mathcal{L}$).
The parent sampling procedure is described in \ref{appdx_parentdist}.

\section{Relation to Causal Inference and Counterfactual Reasoning}
\label{sec:causal-appendix}

Although \methodName leverages a causal structure on the transition dynamics, and although it is possible that the models used by \methodName could be used for a limited form of causal inference (described below), the actual \methodName algorithm does not do Pearl-style ``counterfactual reasoning'' \cite{pearl2009causality} when sampling new transitions. This is because the generic algorithm presented in Algorithm \ref{algo:mocoda-fmdp} uses the parent model $P_\theta(s,a)$ and dynamics model $P_\phi(s'|s,a)$ together to sample entire transitions \textit{de novo}, whereas Pearl-style counterfactuals ask ``what if $X'$ happened instead of $X$ (given $Y$ was observed)?''. Answering this latter question in the SCM framework involves inference over exogenous noise variables, and typically results in a partial relabeling of the data with the counterfactual result. In \methodName, there are no explicit noise variables, and the noise is implicit in the parent and dynamics models.

Nevertheless, one could do a form of Pearl-style counterfactual reasoning using the \methodName models by considering counterfactual ``what if'' questions for a subset of all parent variables, rather than all variables simultaneously. In this case, the parent model could be used to conditionally resample any unspecified parent variables (conditioning the remaining variables on factual observations), and the dynamics model could be used to resample any affected causal mechanisms (keeping the rest of the transition, and therefore any noise implicit in the rest of the transition, fixed).

When carrying out causal inference or counterfactual reasoning, a central concern is \emph{identifiability}: under what assumptions are inferences or counterfactual samples produced by an algorithm said to be unique?
Demonstrating identifiability, say of counterfactual transitions, typically requires the introduction of further assumptions over the structural functions themselves, which can be realized in practice through the use of specialized network architectures (which we have not employed in our experiments).
For example, \citet{oberst2019counterfactual} extended the monotonicy assumption developed for binary outcomes \citep{pearl1999probabilities} to categorical transition dynamics, enabling identifiability analyses to discrete MDPs.
\citet{lu2020sample} applied SCM dynamics to data augmentation in continuous sample spaces, and discussed the conditions under which the generated transitions are uniquely identifiable counterfactual samples.
Without appealing to such approaches, we note that the discrete setting used for analysis in Section \ref{sec:theory} could be further constrained by assuming deterministic and invertible dynamics, which would yield identifiable counterfactual sampling using the \methodName models. 
However, these restrictive assumptions preclude most practical settings.

\end{document}